\documentclass[runningheads]{llncs}
\usepackage{graphicx}
\usepackage[utf8]{inputenc} 
\usepackage[T1]{fontenc}    
\usepackage{hyperref}       
\usepackage{url}            
\usepackage{booktabs}       
\usepackage{amsfonts}       
\usepackage{nicefrac}       
\usepackage{microtype}      

\usepackage{amsmath,amssymb}
\usepackage{amsfonts}
\usepackage{algorithm}
\usepackage{algorithmic}
\usepackage{bm}
\usepackage{multirow}
\usepackage{mathtools}
\usepackage{progressbar}
\usepackage{pifont}
\usepackage{latexsym}
\usepackage{cite}

\def\qed{\hfill $\Box$}

\DeclarePairedDelimiter\norm{\lVert}{\rVert}

\def\ffn{FedNNNN}
\def\nwda{NWDA}

\def\bw{\bm{w}}
\def\dw{{\Delta \bw}}

\def\wk{\bw^k}

\usepackage{xcolor}
\hypersetup{
    colorlinks=false,
    citebordercolor=green,
    linkbordercolor=red,
    urlbordercolor=cyan,
}
\begin{document}
\title{FedNNNN: Norm-Normalized Neural Network Aggregation for Fast and Accurate \\Federated Learning}
\titlerunning{FedNNNN: Fast and Accurate Federated Learning}
%
\author{Kenta Nagura\inst{1} \and
Song Bian\inst{1} \and
Takashi Sato\inst{1}}
\authorrunning{K. Nagura et al.}
%
\institute{Graduate School of Infomatics, Kyoto University, Kyoto, 
Japan \email{paper@easter.kuee.kyoto-u.ac.jp}}
\maketitle              
\begin{abstract}
Federated learning (FL) is a distributed learning protocol in which a server 
needs to aggregate a set of models learned some independent clients to proceed the
learning process.  At present, model averaging, known as FedAvg, is one of the most
widely adapted aggregation techniques.  However, it is known to yield the models with
degraded prediction accuracy and slow convergence.  In this work, we find out that
averaging models from different clients significantly diminishes the norm of the update
vectors, resulting in slow learning rate and low prediction accuracy.
Therefore, we propose a new aggregation method called FedNNNN. Instead of
simple model averaging, we adjust the norm of the update vector and introduce
momentum control techniques to improve the aggregation effectiveness of FL.
As a demonstration, we evaluate FedNNNN on multiple datasets and scenarios with different 
neural network models, and observe up to $5.4\%$ accuracy improvement.

\keywords{Federated Learning  \and Distributed Learning \and Deep Learning}
\end{abstract}

\section{Introduction}\label{sec:intro}
The use of neural networks (NN) in applications such as image classification, natural language processing 
and speech recognition have become one of the core infrastructures in modern lives.
As the number of devices in the age of internet of things (IoT) increases, a large amount 
of data can easily be collected to improve the accuracy of the NN models.
On the other hand, due to privacy and computation resource concerns,
federated learning (FL) ~\cite{bib:Brendan2016} is attracting major attentions 
recently. In contrast to conventional single-machine model training, the FL server 
aggregates models that are locally trained by clients, and can thus be considered as a type of 
distributed learning~\cite{bib:Ma2017,bib:Lin2015,bib:Jeffrey2012,bib:Chilimbi2014}.
As model aggregation involves much less computations and communication bandwidth on the server
compared to direct training, FL is a preferred strategy for the providers of the
machine learning services.

Unfortunately, most existing FL frameworks are either  unrealistic in their protocol
construction, or impractical in terms of their prediction accuracy.
In particular, the model averaging technique proposed by federated averaging 
(FedAvg)~\cite{bib:Brendan2016}, adopted in most FL frameworks, result 
in about 15\% accuracy degradation over non-IID (non independent and 
identically distributed) datasets~\cite{bib:Yue2018}. Notable improvements
over FedAvg include~\cite{bib:Yue2018,bib:Tian2018,bib:Moming2019,bib:Shoham2019}.
We defer a more detailed discussion of related works to Section~\ref{sec:relatedworks},
but point out that many existing works failed to improve the accuracy of FL across datasets.
Those techniques that are successful in addressing the accuracy
degradation problem generally rely on additional data 
sharing~\cite{bib:Yue2018,bib:Moming2019,bib:Shoham2019}, which defeats the
original purpose of FL in many practical applications.

In this paper, we propose {\ffn}, a norm-based neural network aggregation technique. 
In particular, we first propose norm-based weight divergence analysis (\nwda) to visualize
and explain how FL proceeds, particularly in the presence of the FedAvg algorithm. 
We then introduce {\ffn}, and show that we can improve the
prediction accuracy of FL up to 5\%. The main contributions of this work are summarized
as follows.
\begin{itemize}
  \item {\bf{{\nwda} and The Updating Direction Divergence Problem}}:
    We point out that local weight updates in different (diverging) directions on 
    the clients result in small update norms, and we call this the weight updating direction divergence problem (WUDD). 
    WUDD is identified as the most 
    important reason for the slow convergence speed and degraded accuracy performance
    for FedAvg and the related works that build on FedAvg.
  \item {\bf{Norm-Normalized Aggregation}}:
    We propose a normalization technique that targets
    on solving the WUDD problem. In the technique, we apply a simple 
    normalizing factor during model 
    aggregation on the server with an additional momentum term to force and accelerate
    the learning process over the communication rounds.
  \item {\bf{Improved Accuracy with Negligible Overheads}}:
    By conducting rigorous experiments with the proposed technique, we observe accuracy 
    improvements across datasets compared to the-state-of-the-art FL 
    techniques with extremely small computational overheads.
\end{itemize}

The rest of this paper is organized as follows.
In Section~\ref{sec:prelim}, we explain FL and FedAvg in detail, and discuss
some related works.
In Section~\ref{sec:analysis}, we introduce our norm-based FL analysis method {\nwda}, 
and in Section~\ref{sec:nnnn}, our norm-based aggregation method, FedNNNN is formulated.
In Section~\ref{sec:experiments}, our method is evaluated with various datasets experimentally,
and we conclude our work in Section~\ref{sec:conclusion}.

\section{Preliminaries and Related Works}\label{sec:prelim}
\subsection{Federated Learning}\label{sec:FL}
Federated Learning (FL)~\cite{bib:Brendan2016} is a distributed learning protocol 
that enables one to train a model with a massive amount of data obtained by IoT devices or 
smartphones without expensive training on a centralized server.
Instead of collecting the data and training the model on a single machine, 
FL server only combines models locally trained by edge devices to proceed the learning process, 
and this combination procedure is known as model aggregation.
The properties of FL is extensively studied over the past few 
years~\cite{bib:Brendan2016,bib:Jakub2016,bib:Robin2017,bib:Bonawitz2017,bib:Eugene2018,bib:Wang2018,bib:Yue2018,bib:Tian2018,bib:Nishio2019,bib:Felix2019,bib:Shoham2019,bib:Mohri2019},
where we see discussions on aspects of FL such as communication efficiency~\cite{bib:Jakub2016}, 
adaptation to heterogeneous systems~\cite{bib:Nishio2019}, performance
over non-standard data 
distributions~\cite{bib:Yue2018,bib:Tian2018,bib:Moming2019,bib:Felix2019,bib:Shoham2019,bib:Mohri2019}, 
security properties~\cite{bib:Robin2017,bib:Bonawitz2017,bib:Eugene2018}, and many more.
In this work, we focus on improving the prediction accuracy of FL over non-IID distributions,
which is one of the main problems associated with existing FL frameworks.

\subsection{Federated Averaging}
In this section, we outline the Federated Averaging (FedAvg)
framework~\cite{bib:Brendan2016}.
The overview is depicted by Figure~\ref{fig:fedavg}.
\begin{figure}[t]
    \begin{tabular}{cc}
      \begin{minipage}[t]{0.48\linewidth}
        \centering
        \includegraphics[width=\linewidth]{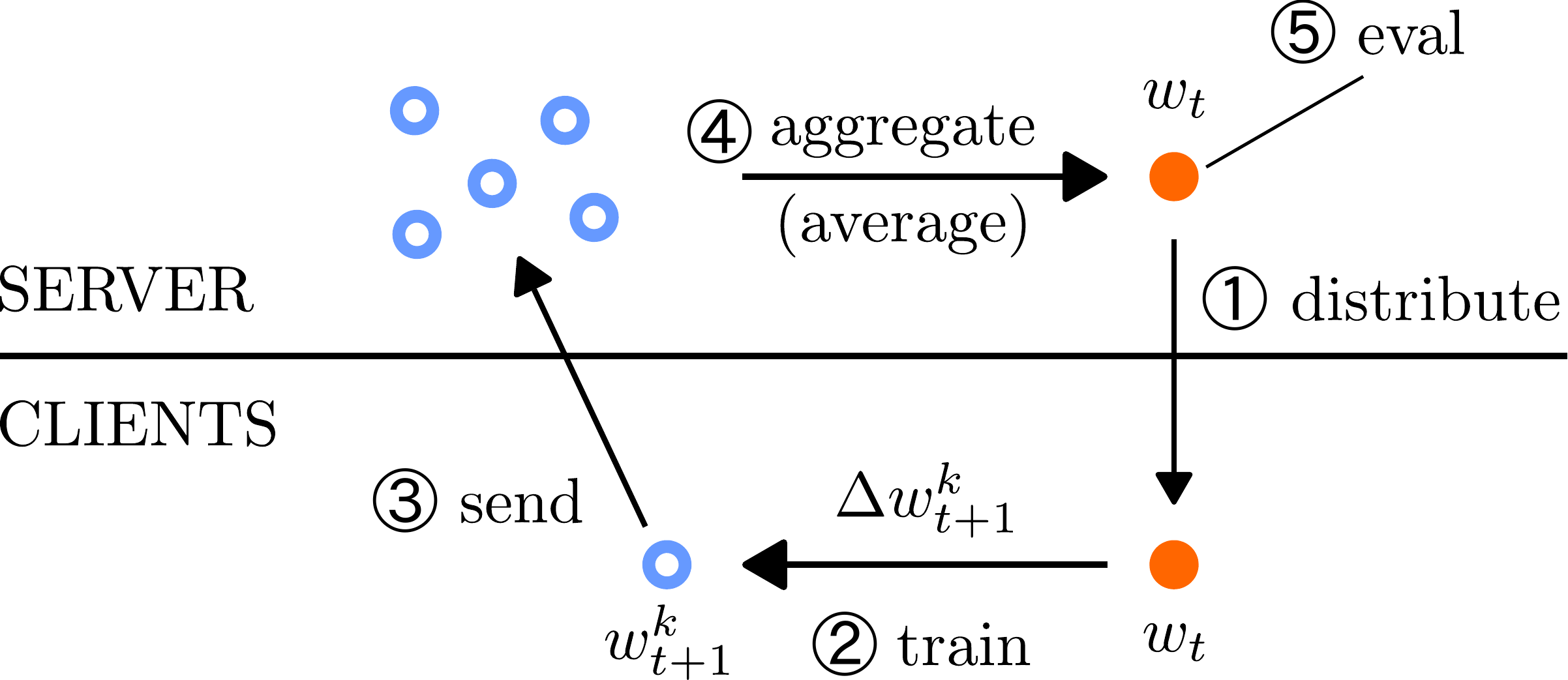}
        \caption{Conceptual illustration of the FedAvg protocol.}
        \label{fig:fedavg}
      \end{minipage} 
      \hspace{0.2cm}
      \begin{minipage}[t]{0.48\linewidth}
        \centering
        \includegraphics[width=0.55\linewidth]{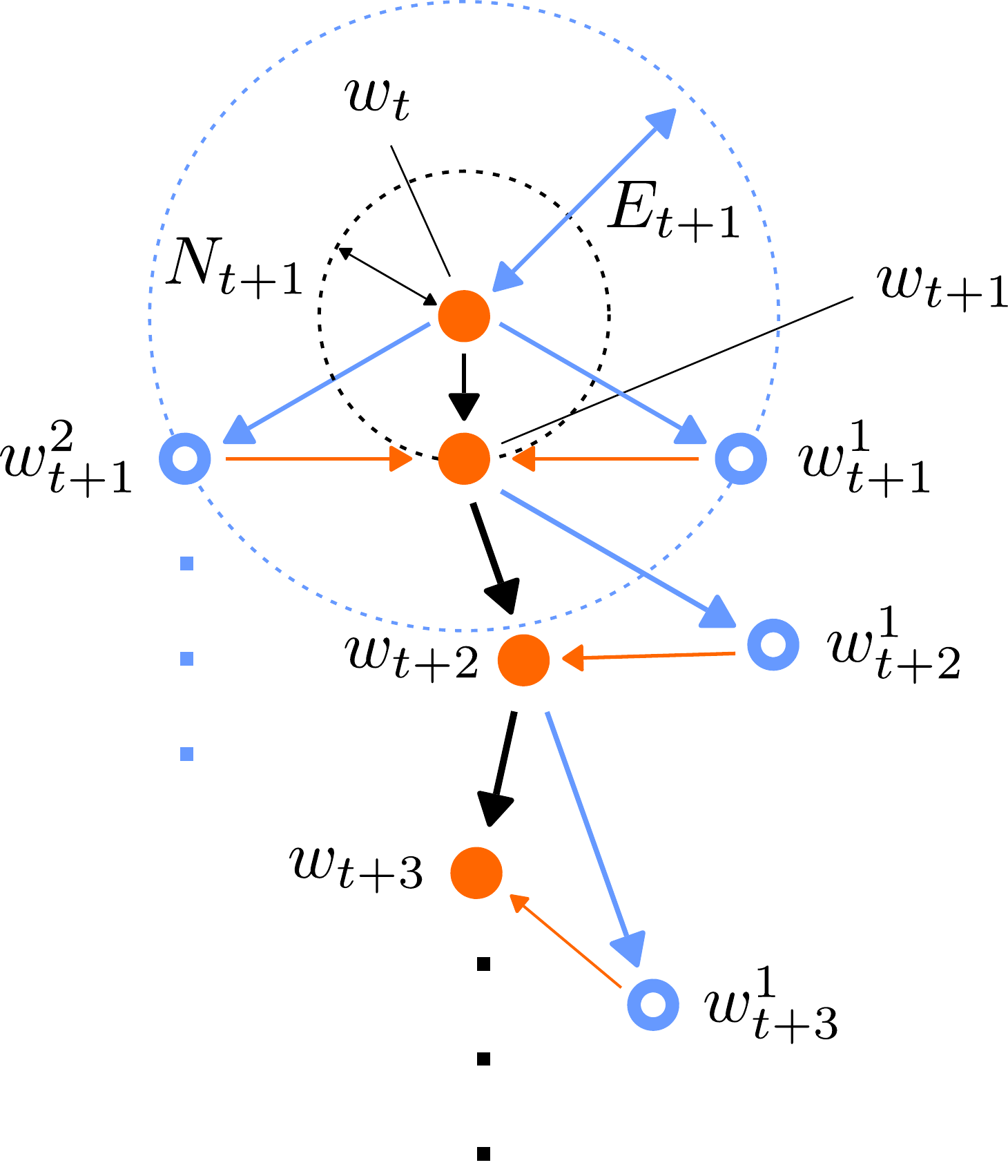}
        \vspace{0.4cm}
        \caption{The illustration of the weight updating direction divergence (WUDD) problem in FedAvg.}
        \label{fig:fedavg_move}       
      \end{minipage}
    \end{tabular}
\end{figure}
In FedAvg, we assume that clients are in possession of the data, 
and a server aggregates the models from the clients. The learning
is proceeded by the following steps.
\begin{itemize}
  \item Step {\large\ding{192}}: Let the total number of clients
  to be $K$. The server first picks $m=\max(C\cdot K, 1)$ clients out of the total $K$ clients
  for some real number $C\in [0, 1]$. In the first round, the server initialize a model
  $\bw_{0}$ and distribute the model to the selected $m$ clients.
  Otherwise, the server distributes the aggregated model $\bm w_t$. 
  We consider the model distribution to be the start of the $(t+1)$-th round
  of communication.
  \item Step {\large\ding{193}}: Upon receiving the server model, each client 
  locally trains the their own models using $\bm w_t$ as the initial model for $\mathcal{E}$ epochs.
  The local model trained on the $k$-th client in the $(t+1)$-th communication round
  is referred to as $\bw^{k}_{t+1}$.
  \item Step {\large\ding{194}}: Client $k$ returns its trained model $\bm w^k_{t+1}$ to
  the server.
  \item Step {\large\ding{195}}: Upon receiving $m$ locally trained models from
  the clients, the server aggregates the models by simply taking the weighted average as
  \begin{equation}
    \label{eq:fedavg}
    \bm w_{t+1}\leftarrow \sum_{k=1}^{m}\frac{n_k}{n}\bm w^k_{t+1},
\end{equation}
where $n_{k}$ is the size of the dataset on the $k$-th client, and we have 
that $n=\sum_{k=1}^m n_k$.
  \item Step {\large\ding{196}}: The server evaluates a unified model with test data.
     Additional learning steps can be carried out by repeating the steps from {\large\ding{192}} to
     {\large\ding{196}}.
\end{itemize}

\subsection{Improvements on FedAvg}\label{sec:relatedworks}
As mentioned, the averaging aggregation utilized in FedAvg results in significant
accuracy degradation during server model evaluation if the clients possess non-IID 
datasets. A line of 
works~\cite{bib:Yue2018,bib:Tian2018,bib:Moming2019,bib:Felix2019,bib:Shoham2019,bib:Mohri2019} 
are proposed to address this problem. Here we give a brief review on the existing methods.

The work in~\cite{bib:Yue2018} is one of the first to point out that non-IID datasets
result in drastically different local models, and these models become difficult to
aggregate with simple averaging. However,~\cite{bib:Yue2018} only proposes to share
auxiliary datasets to each client so that each client obtains a more IID dataset.
\cite{bib:Tian2018} tries to adjust the loss function to improve upon~\cite{bib:Yue2018}, but our analysis
shows that the improvements are not consistent across datasets. In 
Section~\ref{sec:nwda_related}, we take a deeper look at the exact reason for the
degraded accuracy of FL over non-IID datasets.

Other optimization approaches include~\cite{bib:Moming2019,bib:Felix2019,bib:Shoham2019}. In~\cite{bib:Moming2019}, clients are grouped based on the label
distribution of their datasets. 
In~\cite{bib:Felix2019},
the server aggregates the client models whenever a single weight update occurs on some
particular client. 
These approaches clearly incur a large amount of communications 
between the server and the clients. 
FedCurv~\cite{bib:Shoham2019}, which imposes complex loss function on clients, also
induces communication overheads.
In summary, existing works generally require additional datasets or complex communication protocols
to directly solve the non-IID problem on the data level. In what follows, we introduce quantitative 
analyses and normalization-based techniques to mitigate the impact of non-IID datasets without
relying on auxiliary datasets or complex protocol modification.

\section{Aggregating Divergent Weights in Federated Averaging} \label{sec:analysis}
We propose a norm-based weight divergence analysis ({\nwda}) technique to 
visualize how learning proceeds in FedAvg-based aggregation techniques over 
non-IID datasets in this Section. 

\subsection{Norm-based Weight Divergence Analysis}\label{sec:nwda}
We start by defining a per-client $\dw^{k}_{t+1}$ for the $(t+1)$-th communication round
\begin{equation}
\label{eq:diff}
\Delta \bm w_{t+1}^k = \bm w^k_{t+1} - \bm w_t
\end{equation}
where $\bw_{t}$ is the weight vector (i.e., neural network model) distributed to each 
client in the start of the $(t+1)$-th communication round, and $\bw^k_{t+1}$ is the locally learned
weights that will be returned to the server in the $t+1$-th round. Hence, we can
simply interpret the weight difference $\dw^{k}_{t+1}$ as the amount of learning 
proceeded in a single round of communication for client $k$.
 
Using Equation~\eqref{eq:diff}, we can re-formulate Equation~\eqref{eq:fedavg} the aggregation 
procedure in FedAvg as
\begin{equation}\label{eq:norm_form}
    \bm w_{t+1}\leftarrow \bm w_t + \sum_{k=1}^m \frac{n_k}{n}\Delta \bm w^k_{t+1}.
\end{equation}
In other words, since all clients share the same $\bw_{t}$, we can
express the weight-averaging procedure
in FedAvg as the sum of the distributed model $\bw_{t}$ and the averaged sum of 
the local updates from each of the clients.

To quantitatively assess the impact of the updating vector, we define a pair
of real scalars $N_{t+1}, E_{t+1}\in\mathbb{R}$ using the $L_{2}$ norm $\norm*{.}$ as follows
\begin{eqnarray}
    N_{t+1}&:=\norm*{\sum_{k=1}^m \frac{n_k}{n}\Delta \bm w^k_{t+1}}&\label{eq:norm}\\
    E_{t+1}&:= \sum_{k=1}^m \frac{n_k}{n}\norm*{\Delta \bm w^k_{t+1}}.&\label{eq:enorm}
\end{eqnarray}
Here, $N_{t+1}$ in Equation~\eqref{eq:norm} is the distance
that the server model $\bw_{t}$ moved from the $t$-th round to 
the $(t+1)$-th round. Whereas, $E_{t+1}$ in Equation~\eqref{eq:enorm} is the 
average of the norms of each local weight updating vector $\norm{\dw^{k}_{t+1}}$ 
on the $k$-th client. Consequently, we can think of $N$ as the amount of 
server model updates, and $E$ as the (average) amount of local updates in clients.

The following proposition expresses the relationship between $N$ and $E$.

\begin{proposition}
\label{prop:inequality}
The following inequality holds
\begin{equation}
    N_{t+1}\leq E_{t+1}
\end{equation}
\text{for all $t\in (0, 1, \cdots)$}
\end{proposition}
As Proposition~\ref{prop:inequality} follows trivially from a recursive application
of the Pythagorean inequality, we leave a formal proof to the appendix. The main
idea behind the proposition is that, while each client locally proceeds the learning
process by an average of $E_{t+1}$, after model aggregation, the server only learns
by $N_{t+1}$, which is guaranteed to be less than $E_{t+1}$ by 
Proposition~\ref{prop:inequality}.

The previously described learning behavior can be better illustrated through 
Figure~\ref{fig:fedavg_move}. The important observation here is that, since 
each client learns locally without online communication, their updating direction
diverges. As the server averages these local models in FedAvg, by the
formulation of Equation~\eqref{eq:norm_form}, the local models tend to cancel 
each other out, resulting in a extremely small updating vector on the server
(i.e., $N\ll E$). We refer to this phenomenon as the weight updating direction 
divergence (WUDD) problem.
\begin{figure}[t]
    \begin{tabular}{cc}
      \begin{minipage}[t]{0.48\linewidth}
        \centering
        \includegraphics[width=\linewidth]{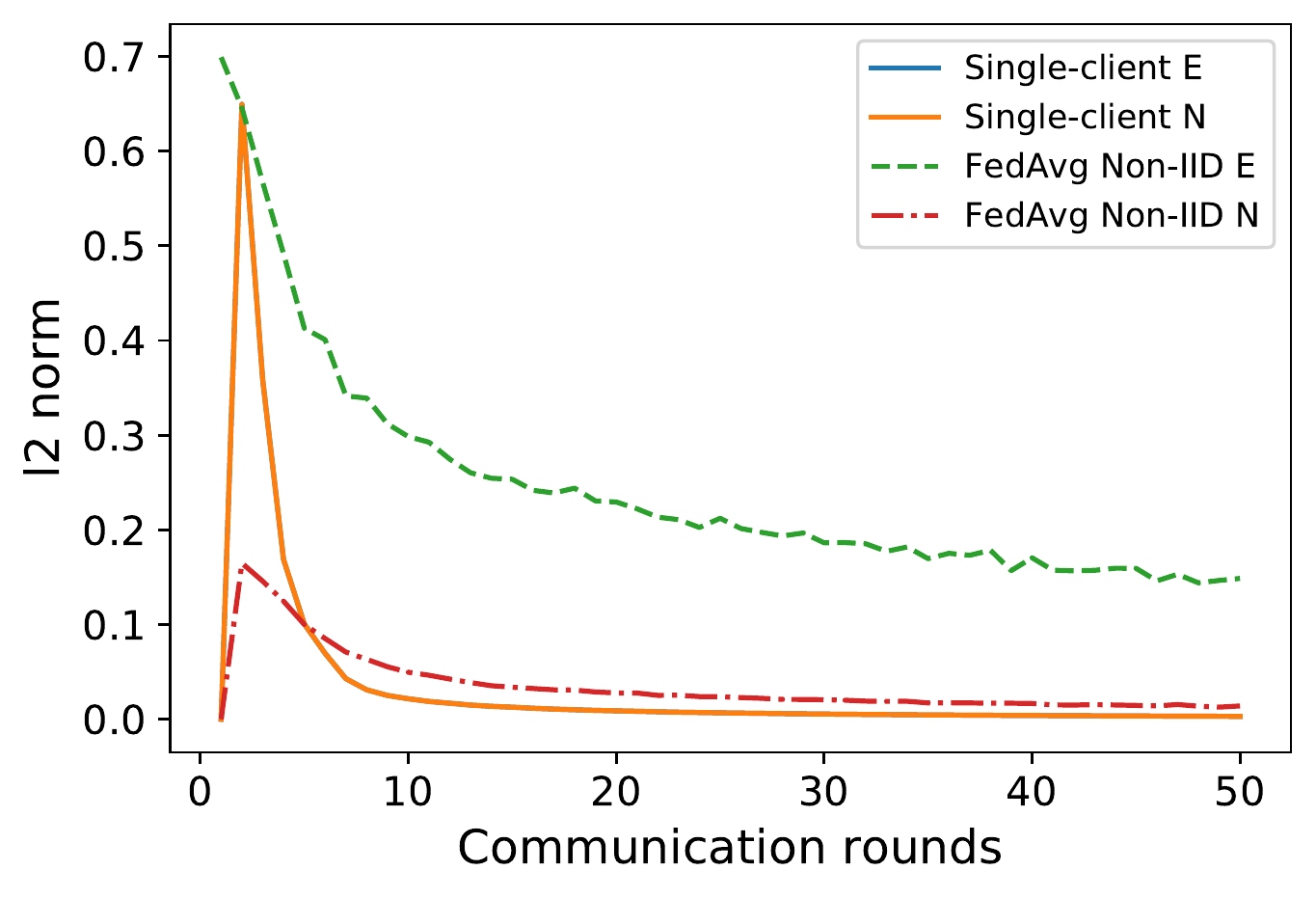}
        \caption{The change in the size of $N$ and $E$ over the communication rounds $t$.}
        \label{fig:L2full} 
      \end{minipage} 
      \hspace{0.2cm}
      \begin{minipage}[t]{0.48\linewidth}
        \centering
        \includegraphics[width=\linewidth]{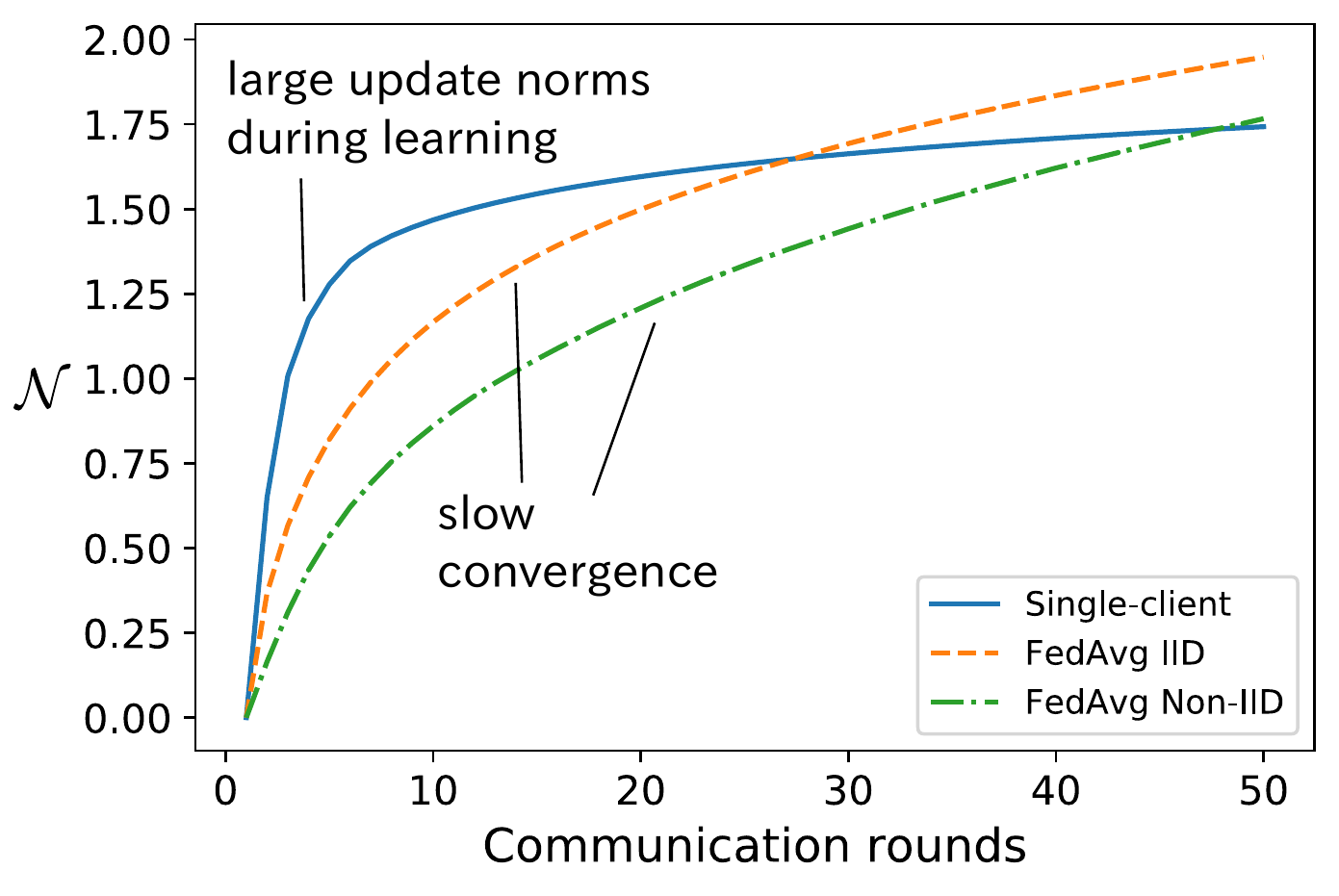}
        \caption{The integrated norm $\mathcal{N}$ of average weight difference 
        vectors $N$ over the communication rounds $t$.}
        \label{fig:L2full_acc} 
      \end{minipage}
    \end{tabular}
\end{figure}


As a demonstration of the {\nwda} technique, we show the calculated $N$ and $E$ 
values using MNIST dataset. 
Figure~\ref{fig:L2full} show how the sizes of $N$ and $E$ change over
the communication rounds (i.e., the learning epochs). Initially, the
NN tries to learn the dataset through a series of weight updates, and
we observe reasonably large updating norms. Since MNIST is a small dataset,
on the single-client case, the learning soon converges, and both $N$ and
$E$ become extremely small due to the learning convergence (note that
$N=E$ for the single-client case). In contrast, while $E$ is large for
non-IID FedAvg, due to the WUDD problem, $N$ is relatively small throughout
the communication rounds. We also emphasize on the fact that, as the
learning proceeds, $N$ becomes convergent (extremely small $N$) while $E$
remains large. This indicates that some learning updates are still available,
but cannot be learned (or aggregated) by the server through the simple 
averaging technique in FedAvg.

To further study the norm properties of single-client, IID and non-IID
FedAvg, we plot $\mathcal{N}$, the integral of the amount of norm updates
on the server with respect to $t$, i.e., 
$\mathcal{N}=\int_{\tau=0}^{t}\norm*{\bm w_{\tau+1} - \bm w_{\tau}}$ 
in Figure~\ref{fig:L2full_acc}. 
The single-client case shows a clear
trend of fast and convergent learning curve, while both IID and 
non-IID FedAvg do not show clear signs of convergence across the 
communication rounds.

\subsection{Applying {\nwda} in Existing Works}\label{sec:nwda_related}
In Section~\ref{sec:nwda}, we analyzed the learning process of 
FL under FedAvg  as a series of aggregations of vectors representing 
weight updates $\dw^k_{t+1}=\wk_{t+1}-\bw_t$. 
Based on our analysis, we can say that the average of the weights across clients
(i.e., $N_{t}$) obtains a reasonable large norm only when all clients 
update in the same direction. If the weight updates diverge (i.e., the WUDD problem),
learning cannot be executed effectively and efficiently in a simple FedAvg
setting.

While no existing works explicitly 
derive a norm analysis, some techniques~\cite{bib:Yue2018,bib:Moming2019,bib:Felix2019,bib:Shoham2019} 
try to implicitly address the WUDD problem.
For example, as mentioned Section~\ref{sec:relatedworks},
it is shown that if additional datasets (with different class labels) 
are distributed to the clients along to be learned along with the local datasets
on the client, the accuracy improves~\cite{bib:Yue2018}. With {\nwda}, we can
simply interpret this method as reducing the divergence between the updating
weight vectors across independent clients by introducing common datasets.
Based on the observation from~\cite{bib:Yue2018}, in~\cite{bib:Tian2018}, FedProx 
is proposed to reduce the difference between client models. 
FedProx proposes to add a normalization term $\frac{\mu}{2}\norm*{\bm w^k-\bm w_t}^2$
in the loss function on each client as
\begin{equation}
\label{eq:fedprox}
\min_{\bm w^k\in \mathbb{R}^d} Q(\bm w^k) + \frac{\mu}{2}\norm*{\bm w^k-\bm w_t}^2
\end{equation}
where $Q$ is the original loss, and $\mu$ is a hyperparameter. The basic idea of FedProx
is to penalize weight updates that largely increase the distance (i.e., norm) between
the locally learned model $\bm w^{k}$ and the distributed model $\bm w_{t}$. 
Essentially, by 
introducing a penalty term in the loss function, FedProx only reduces the size of 
$E$, and the diverging direction problem is not addressed.
However, $N$ is what actually matters to the evaluation accuracy of the
aggregated model. As long as $N$ is much less than the required amount of 
weight updates to fully perceive a dataset, the aggregated model cannot obtain
a reasonable level of accuracy.
\section{Norm-Normalized Neural Network Aggregation}\label{sec:nnnn}
In this section, we formally present the proposed norm-based aggregation (FedNNNN) technique
that improves the prediction accuracy of FL without incurring large computational
or communication overheads. We first introduce the slightly modified FL protocol
adopted in {\ffn}, and then demonstrate that {\ffn} is more efficient and 
effective in improving the accuracy of FL, especially over non-IID datasets.

\subsection{FedNNNN: The Protocol}
\begin{figure}[t]
    \begin{tabular}{cc}
      \begin{minipage}[t]{0.48\linewidth}
        \centering
        \includegraphics[width=\linewidth]{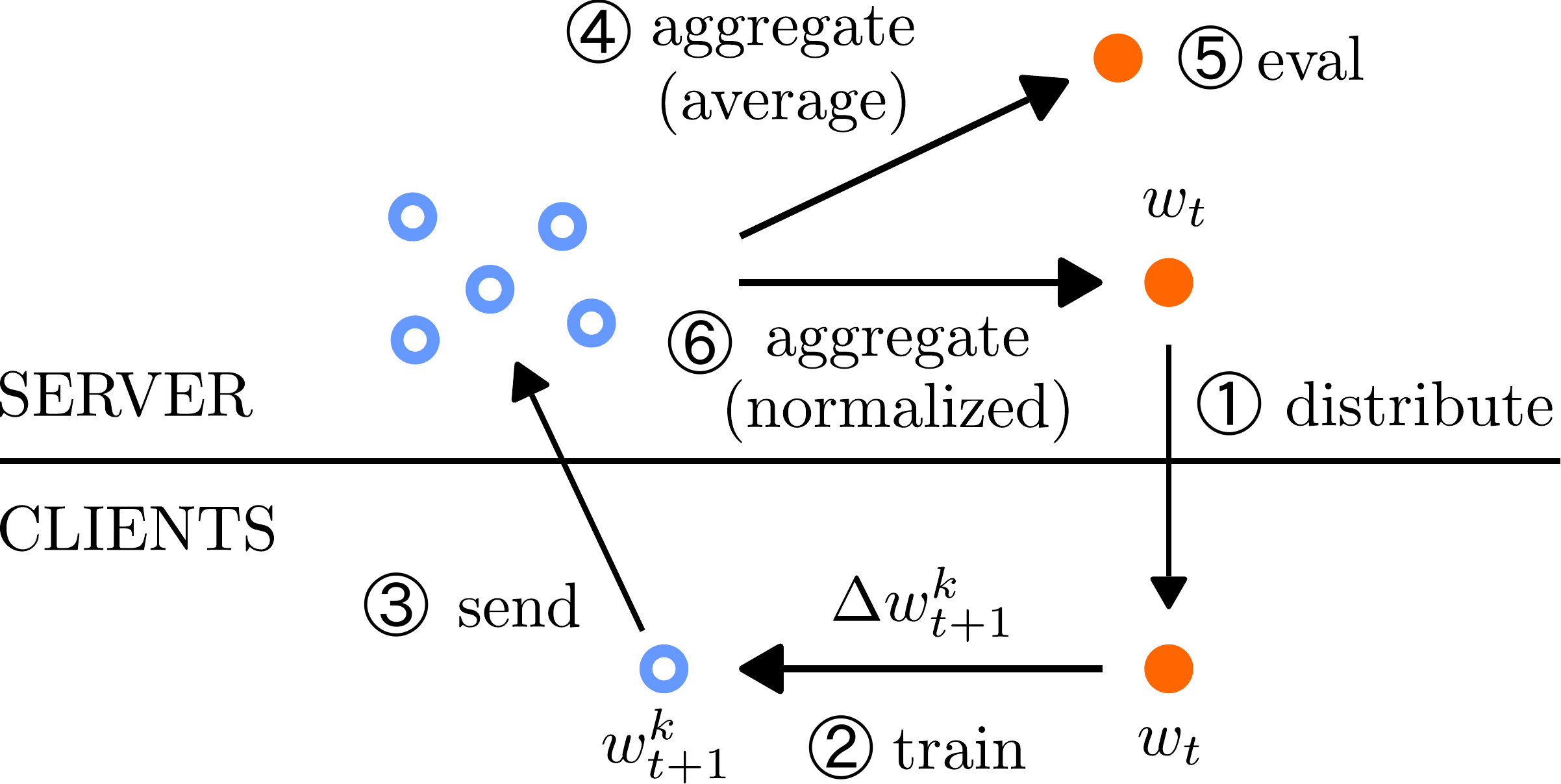}
        \caption{The proposed federated learning protocol for adopting our norm-normalized neural network aggregation technique.}
        \label{fig:fedproposed}
      \end{minipage} 
      \hspace{0.2cm}
      \begin{minipage}[t]{0.48\linewidth}
        \centering
        \includegraphics[width=\linewidth]{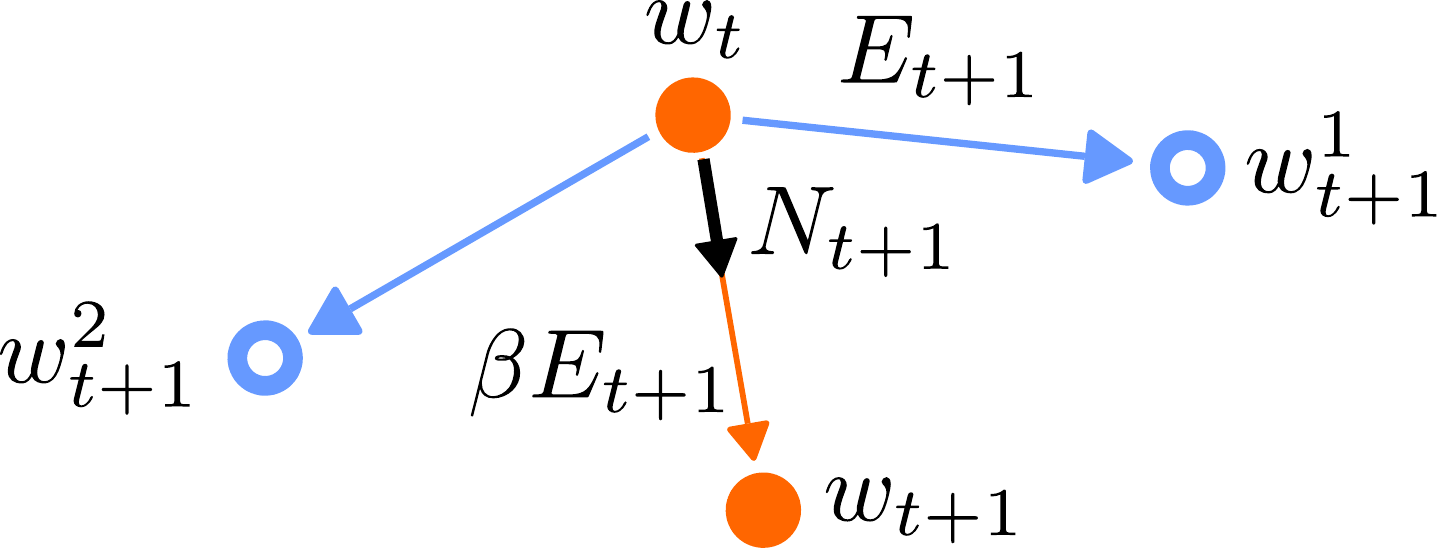}
        \caption{A conceptual illustration of the normalized aggregation technique.}
        \label{fig:fednorm}
      \end{minipage}
    \end{tabular}
\end{figure}

The main difference between the {\ffn} protocol and the conventional FL
protocol is that, we separate the evaluation model and the aggregated model.
As shown in Figure~\ref{fig:fedproposed}, after the clients locally produce
their models, there are two aggregation steps.
\begin{itemize}
  \item {\large\ding{195}} Average Aggregation: The same aggregation in FedAvg is
  adopted to produce a single model that is capable of performing 
  highly-accurate inference based on the aggregated models.
  \item {\large\ding{197}} Normalized Aggregation: The proposed normalized aggregation
  technique is used to resolve the WUDD problem where $N\ll E$. 
  More details on the complete {\ffn} aggregation is provided
  in Section~\ref{sec:distagg}.
\end{itemize}
All other steps remain the same as in the conventional FL protocol. 
Lastly, we note that the proposed normalized aggregation technique 
is only good as a foundation of learning. In fact, due to our modification
on the norm of the weights (as seen in the next section), prediction accuracy
on the aggregated model from step {\large\ding{197}} remains extremely poor.

\subsection{FedNNNN: The Aggregation}\label{sec:distagg}


As discussed in Section~\ref{sec:nwda_related}, according to our {\nwda}, 
the main reason that existing FL schemes fall short on non-IID datasets is 
that averaging models with divergent directions significantly reduces the 
updating norm of the aggregated model. Hence, in {\ffn}, we propose to normalize
the updating norm as
\begin{equation}
    \label{eq:fednorm}
     \bm w_{t+1} \leftarrow \bm w_t + \beta \frac{E_{t+1}}{N_{t+1}}\sum_{k=1}^m\frac{n_k}{n}\Delta \bm w_{t+1}^k .
\end{equation}
Consequently, the norm of the normalized updating weight vector becomes
\begin{eqnarray}
    \label{eq:n_fednorm}
    \norm*{\bm w_{t+1}-\bm w_t}&=&\norm*{\beta \frac{E_{t+1}}{N_{t+1}}\sum_{k=1}^m\frac{n_k}{n}\Delta \bm w_{t+1}^k}
        =\beta \frac{E_{t+1}}{N_{t+1}}\cdot N_{t+1}=\beta E_{t+1}.
\end{eqnarray} 
In order to control the impact of the norm normalization, we introduce a
hyperparameter $\beta$.
In the experiment, we empirically decide on the value of $\beta$.

A conceptual illustration of our normalized aggregation 
is depicted in Figure~\ref{fig:fednorm}. We assume that the 
weighted average $\sum_{k=1}^m\frac{n_k}{n}\Delta \bm w_{t+1}^k$ gives us a 
reasonably correct updating direction, but with
insufficient norm. We amplify the norm of the updating vector in that particular
direction with the size of $E_{t+1}$. The reason that $E_{t+1}$ works as a proper
normalization factor stems from our observation in Figure~\ref{fig:fedavg_move}, 
where we see that the size of $E$ remains large after $N$ converges in the non-IID
case. As a result, the amount of unfinished learning can be expressed through the size
of $E$, and if $E$ approaches to zero, we can safely conclude that no significant client
update is available anymore. Note that when $N$ is extremely small, e.g., 
0 or close to 0, we do not normalize the updating vector and return $\bw_{t}$
as is.

While the proposed normalized aggregation technique described above is successful in 
advancing the learning process in a given direction, we still find that, due to
its democratic nature, FL is susceptible to local extrema. To avoid the weights being 
stuck in a 
low-accuracy state, we add a momentum term to the aggregation process that is reminiscent 
to the momentum stochastic gradient descent (SGD) method~\cite{bib:Qian1999}
\begin{eqnarray}
    \label{eq:momagg1}
    \bm d_{t+1}&\leftarrow& \gamma \bm d_t + \sum_{k=1}^m\frac{n_k}{n}\Delta \bm w_{t+1}^k\\
    \label{eq:momagg2}
    \bm w_{t+1}&\leftarrow& \bm w_t + \bm d_{t+1}.
\end{eqnarray}
Here, $\gamma$ is the same hyperparameter as in momentum SGD that expresses the amount
of past weight changes to be memorized. Equation~\eqref{eq:momagg1} defines
the $(t+1)$-th round momentum $\bm d_{t+1}$ to be a scaled version of the average weight 
updates (note $\bm d_{0}=0$), and these weight updates accumulate as a driving force
to prevent the learning from being stuck in non-optimized states. Hence, the
aggregation formula of our normalized aggregation with momentum adjustment
becomes
\begin{eqnarray}
    \label{eq:fed4n_mom}
    \bm d_{t+1}&\leftarrow& \gamma \bm d_t + \beta \frac{E_{t+1}}{N_{t+1}}\sum_{k=1}^m\frac{n_k}{n}\Delta \bm w_{t+1}^k\\
    \label{eq:fed4n_update}
    \bm w_{t+1}&\leftarrow& \bm w_t + \bm d_{t+1},
\end{eqnarray}
and we leave the complete description of the entire protocol to the appendix.
\section{Experiments}\label{sec:experiments}
In this section, we present the effectiveness of our proposed FedNNNN
through the experiments on 2 image datasets, MNIST and CIFAR10~\cite{bib:CIFAR} and
4 different FL settings for data distributions.
According to the size of each dataset, different CNN models have been implemented using the PyTorch framework.
In this work, we compare our technique to FedAvg and FedProx, as they have 
similar computation and communication characteristics to {\ffn}.
\subsection{Settings}
We use the MNIST dataset that contains 10 classes of grey-scale handwritten images of numbers,
and the CIFAR10 dataset which classifies RGB images into 10 classes. 
The parameters used in the following experiments are summarized in 
Table~\ref{tab:params}, where the notations are
commonly used in most FL methods.
Total Rounds represents the total number of communications between the server and the clients. 
$\eta$ and $\lambda$ are learning rate and weight decay parameters, respectively.
We used SGD without momentum for optimization.
Due to the limited space, the exact valuations of $\mu$, $\beta$ and $\gamma$ which used in each FL method 
will be listed in the appendix.
We split the dataset into training and test data, where the training data 
are assumed to be inherently held by the clients 
and test data are evaluated at server side.
The hyperparameters $\eta$, $\lambda$, $\mu$, $\beta$, $\gamma$ were chosen to have the
highest accuracy on the test data over the training rounds.

Here, we outline the architecture of CNN models used in this section.
For MNIST, we used a CNN architecture with 2 convolutional layers 
followed by 2 fully connected layers.
For CIFAR10, a CNN with 3 fully connected layers following 6 
convolutional layers with batch normalization layer is used.
Convolutional and fully connected layers of both the MNIST and the 
CIFAR10 architectures have bias terms.
In Table~\ref{tab:params}, BN indicates whether batch normalization is used.
As a separate preprocessing, 
we normalized input images by subtracting the mean and dividing by 
the standard deviation.

\begin{table}[t]
\footnotesize
\caption{Summary of Hyperparameters Used in the Experiment (Common in Scenarios)}
\label{tab:params}
\centering
\begin{tabular}{lcccccccc}
Parameters & Total Rounds & $C$ & $K$   & $B$  & $\mathcal{E}$ & $\eta$ &$\lambda$& BN \\ \hline\hline
MNIST      & 100          & 1 & 100 & 50 & 5 & 0.05 &0      & no            \\ \hline
CIFAR10    & 250          & 1 & 100 & 50 & 5 & 0.05 &$5\times10^{-4}$& yes  \\ \hline
\end{tabular}
\end{table}

\begin{table}[]
\footnotesize
\caption{Summary of Accuracy Results}
\label{tab:accuracy}
\begin{tabular}{lcccccccc}
\multirow{2}{*}{} & \multicolumn{2}{c}{IID - B} & \multicolumn{2}{c}{non-IID - B} & \multicolumn{2}{c}{IID - UB} & \multicolumn{2}{c}{non-IID - UB} \\ 
                  & MNIST           & CIFAR10       & MNIST         & CIFAR10       & MNIST         & CIFAR10       & MNIST         & CIFAR10       \\ \hline\hline
FedAvg            & 99.0            & 81.3          & 98.2          & 72.6          & 99.1          & 83.5          & 93.5          & 56.6          \\ \hline
FedProx           & 99.0            & 81.4          & 98.1          & 74.2          & 99.1          & 83.2          & 87.9          & 55.6          \\ \hline
Norm-Norm         & 99.1            & 81.2          & 99.0          & 76.7          & 99.2          & 83.3          & 97.1          & 55.5          \\ \hline
Momentum          & 99.1            & 83.1          & \textbf{99.2} & 74.3          & 99.1          & \textbf{84.2} & 96.9          & \textbf{62.0} \\ \hline
FedNNNN           & \textbf{99.2}   & \textbf{84.7} & 99.1          & \textbf{76.7} & \textbf{99.2} & 84.0          & \textbf{98.9} & 61.9          \\ \hline
\end{tabular}
\end{table}


\begin{figure}[t]
    \begin{tabular}{cc}
        \begin{minipage}[t]{0.32\linewidth}
        \centering
        \includegraphics[width=\linewidth]{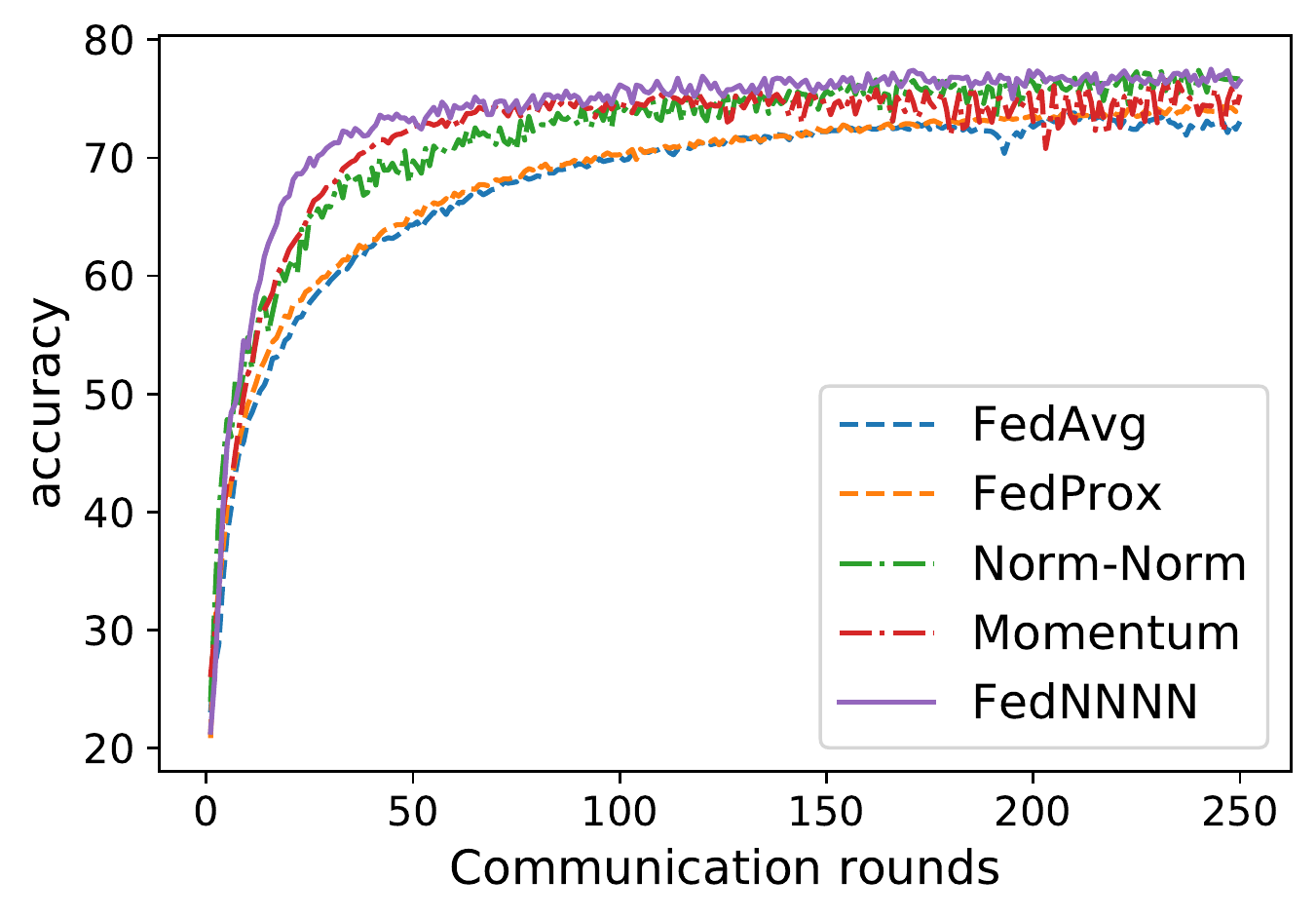}
        \caption{Accuracy curves on the CIFAR10 dataset under non-IID - B condition.}
        \label{fig:cifar10_acc}       
      \end{minipage}
      \hspace{0.1cm}
      \begin{minipage}[t]{0.32\linewidth}
        \centering
        \includegraphics[width=\linewidth]{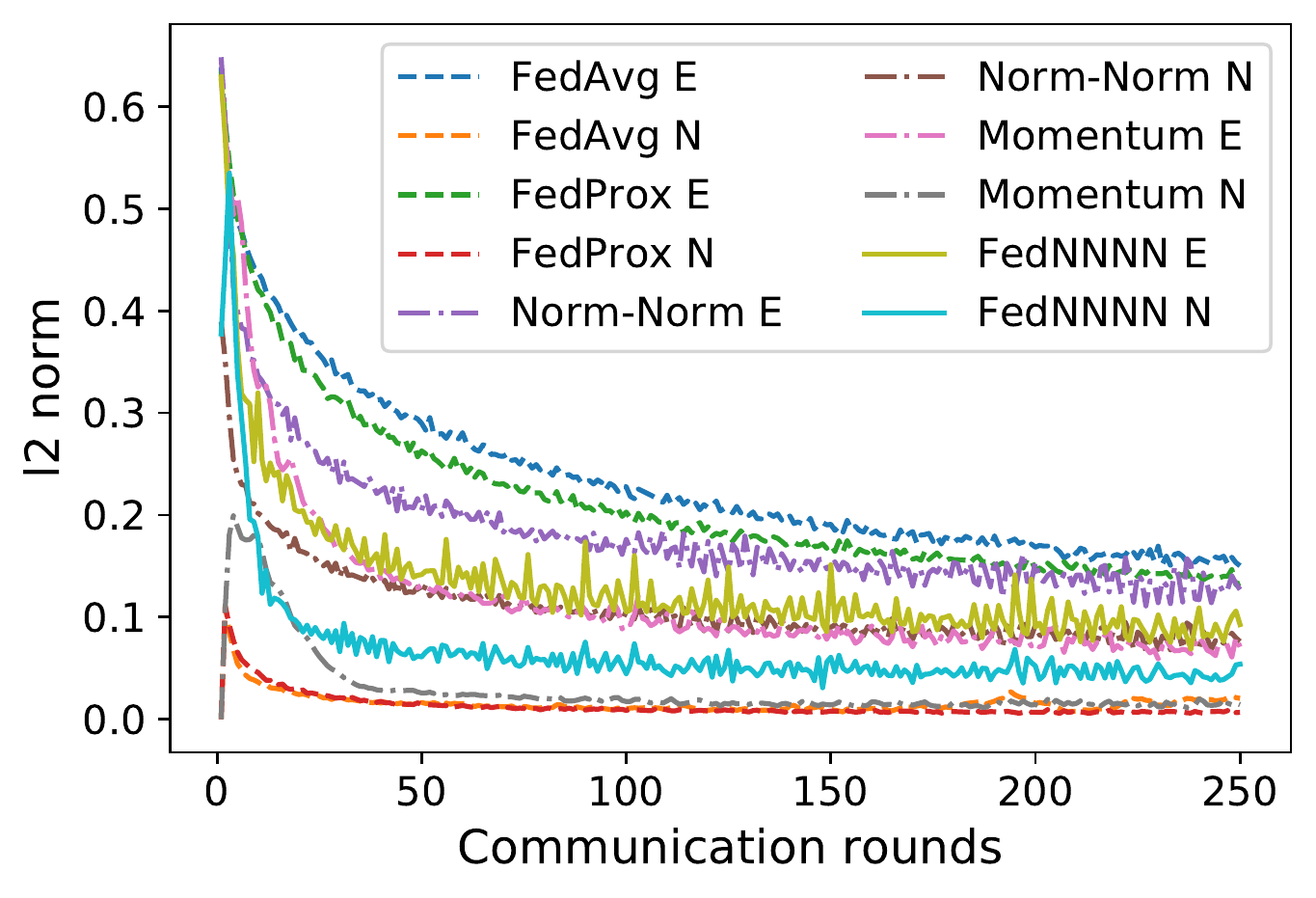}
        \caption{The $N$ and $E$ of fully connected layer calculated on the CIFAR10 dataset
        under non-IID - B condition.}
        \label{fig:cifar10_norm_fc}       
      \end{minipage} 
      \hspace{0.1cm}
      \begin{minipage}[t]{0.32\linewidth}
        \centering
        \includegraphics[width=\linewidth]{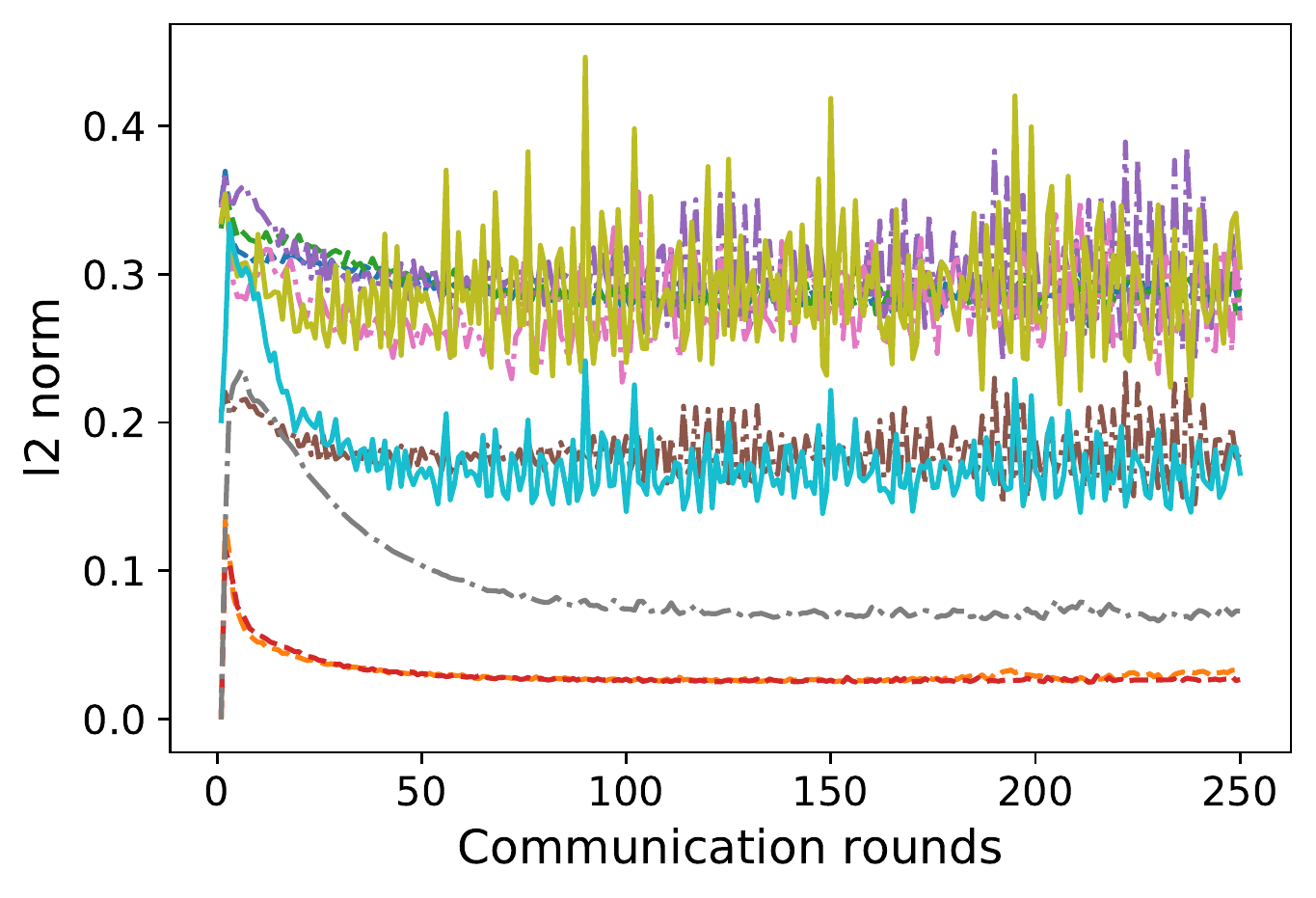}
        \caption{The $N$ and $E$ of convolutional layer calculated on the CIFAR10 dataset
        under non-IID - B condition.}
        \label{fig:cifar10_norm_conv}       
      \end{minipage}
    \end{tabular}
\end{figure}

In this work, we take data distribution into account when performing the
experiments. There are four types of data distribution: i) IID, ii) 
non-IID, iii) balanced, and iv) unbalanced.
For i) IID, clients possess images of all of the defined classes 
in the dataset. In contrast,  for ii) non-IID, clients have images of
only 2 classes (the minimum number of classes that allow meaningful learning).
For the iii) balanced condition, all clients have dataset of the same size.
On the other hand,  iv) unbalanced condition assumes that the
dataset sizes of the clients follows a power-law
distribution, where a small number of clients hold most of the training images.
For privacy reasons, however, we assumed that the server does not know the data size of the
client and set $n_k/n=1/n$ for all $k$.

Experiments are conducted once for each dataset.
For our parameter setting, simulation time for total communication rounds took
an hour for MNIST and 3.5 hours for CIFAR10, respectively, using GeForce GTX 1080 Ti. 

\subsection{Experiment Results}
We evaluated five FL aggregation methods on two datasets with the
aforementioned four different types of dataset distributions in the experiments.
The accuracy comparisons between various experimental conditions 
are summarized in Table~\ref{tab:accuracy}.
In the table, Equation~\eqref{eq:fednorm} is used by Norm-Norm method and
Equation~\eqref{eq:momagg1},~\eqref{eq:momagg2} are used by Momentum. 
B and UB indicate conditions iii) balanced and iv) unbalanced, respectively.

Regardless of the datasets and distributions, the proposed methods achieve
the best prediction accuracy.
Although the momentum appears to be relatively strong, it can be seen that the 
accuracy can be further improved by combining norm normalization.
The changes of accuracy as functions of communication rounds are shown 
in Figure~\ref{fig:cifar10_acc} for non-IID - B CIFAR10.
We can find FedNNNN to significantly improve the convergence rate,
thus speeding up the learning process.

We expect the amount of server update $\norm*{\bm w_{t+1} - \bm w_t}$ in FedNNNN
to be larger than that of FedAvg as discussed in Section~\ref{sec:distagg}.
Figure~\ref{fig:cifar10_norm_fc} and~\ref{fig:cifar10_norm_conv} confirms this
expectation.  
In these figures, $E$ is an average clients update norm 
$\sum_{k=1}^m \frac{n_k}{n} \norm*{\dw^k_{t+1}}$ and 
$N$ is the updated norm on the server side $\norm*{\bm w_{t+1} - \bm w_t}$.
We can clearly observe that the amount of update on the server in FedNNNN is larger 
than the $N$ in FedAvg and FedProx. 
In Figure~\ref{fig:cifar10_norm_fc}, we see that 
$E$ of FedNNNN is smaller than that of FedAvg and 
FedProx, indicating that the weight updates of clients have become smaller because 
of better initial models.
However, in Figure~\ref{fig:cifar10_norm_conv}, $E$ of FedNNNN appears to be
oscillating.
This indicates that the weights of the convolutional layers are less likely to converge
when compared to fully connected layers.
A thorough evaluation of $E$ and $N$ of the fully connected layer and 
convolutional layers for all settings will appear in the appendix.
\section{Conclusion}\label{sec:conclusion}
In this paper, we introduced {\ffn} for the improvement of convergence speed and 
prediction accuracy of FL.
We first defined a norm-based analysis that expresses the amount of model updates by the
$L_{2}$ norm of the sum of the update vectors, and identify that the small size of 
the norm causes slow convergence and accuracy degradation on clients with non-IID 
datasets (the WUDD problem). To solve the WUDD problem, we proposed {\ffn}, 
an aggregation technique that normalize the update vector according to the amount
of unfinished learning. In the experiments, 
we observed that {\ffn} outperforms FedAvg and FedProx, the-state-of-the-art FL frameworks
in terms of both convergence speed and prediction accuracy. In particular, we achieve
up to 5\% accuracy improvement over FedAvg and FedProx on the CIFAR10 dataset.

\bibliography{arxiv}
\bibliographystyle{splncs04}
%
%
%
%
\newpage
\section*{Appendix}
\subsection*{Complete Protocol of FedAvg}
Here, we show the complete protocol of FedAvg.
The server and clients operate as shown in Algorithm~\ref{alg:FedAvg_s} and Algorithm~\ref{alg:FedAvg_c} 
in FedAvg, respectively.
The server executes Algorithm~\ref{alg:FedAvg_s} as follows.
\begin{itemize}
\item Line 1: The server initializes model $\bm w_{0}$ and distribute the model to each client.
\item Line 2: For each round $t=0,1,2...$, the following steps are repeated.
\item Line 3--4: The number of clients to train in this round, $m$, is calculated, and
  the selected set of clients $S_t$ is generated.
\item Line 5: For each client $k\in S_t$:
\item Line 6: The server sends $\bm w_t$ to the clients.
\item Line 7: After the clients train their local models, the server receives trained model 
$\bm w_{t+1}^k$ for each client $k$.
\item Line 9: The server aggregates models averaging them as $\bm w_{t+1}\leftarrow \sum_{k=1}^{m}\frac{n_k}{n}\bm w^k_{t+1}$.
\end{itemize}

On the client side, the following steps are implemented. 
\begin{itemize}
\item Input: The client receives the server model $\bm w_t$ and its label $k$.
\item Line 1: the client
  first obtains some local datasets $\mathcal{P}_k$ (Note that this dataset is not from the server).
  The client splits the entire dataset into mini-batches $\mathcal{B}_{0}, \mathcal{B}_{1}, \cdots, \mathcal{B}_{B-1}$.
  We note the set of the mini-batches as $\mathcal{D}=\{\mathcal{B}_{0}, \mathcal{B}_{1}, \cdots, \mathcal{B}_{B-1}\}$.
\item Line 2: For each epoch $i$ from 1 to $\mathcal{E}$: 
\item Line 3: For each mini-batch $\mathcal{B}\in \mathcal{D}$: 
\item Line 4: The client trains weight as 
$\bm w \leftarrow \bm w - \eta \nabla l(\bm w;\mathcal{B})$.
\item Line 7: The clients send trained weight to the server.
\end{itemize}
\begin{algorithm}[t]
\small
\caption{The algorithmic description of the FedAvg protocol (Server).} 
\label{alg:FedAvg_s}
\begin{algorithmic}[1]
\REQUIRE
\STATE initialize $\bm w_0$
\FOR{each round $t=0,1,2,...$}
\STATE $m \leftarrow \max(C\cdot K,1)$
\STATE $S_t \leftarrow$ (random set of $m$ clients)
\FOR{each client $k\in S_t$ (in parallel)}
\STATE send $\bm w_t$ to the clients
\STATE $\bm w_{t+1}^k \leftarrow$ \textbf{ClientsUpdate($k,\bm w_t$)}
\ENDFOR
\STATE $\bm w_{t+1}\leftarrow \sum_{k=1}^{m}\frac{n_k}{n}\bm w^k_{t+1}$
\ENDFOR
\end{algorithmic}
\end{algorithm}

\begin{algorithm}[t]
\small
\caption{The algorithmic description of the FedAvg protocol (Clients).} 
\label{alg:FedAvg_c}
\begin{algorithmic}[1]
\ENSURE \COMMENT{run on client $k$}
\STATE $\mathcal{D} \leftarrow$ (split $\mathcal{P}_k$ into batch of size $B$)
\FOR{each local epoch $i$ from 1 to $\mathcal{E}$}
\FOR{batch $\mathcal{B}\in \mathcal{D}$}
\STATE $\bm w \leftarrow \bm w - \eta \nabla l(\bm w;\mathcal{B})$
\ENDFOR
\ENDFOR
\STATE return $\bm w$ to server
\end{algorithmic}
\end{algorithm}

\subsection*{Complete Protocol of FedNNNN}
The protocol of FedNNNN is similar to that of the FedAvg protocol. 
First, we note that the client-side procedure remains unchanged, i.e., exactly same as Algorithm~2.
The server-side protocol is also similar but different only for lines 9--12 in Algorithm~\ref{alg:Fed4N}.
\begin{itemize}
\item Line 9--10: $N_{t+1}$ and $E_{t+1}$ are calculated.
\item Line 11--12: the server adds the update vector $\beta
\frac{E_{t+1}}{N_{t+1}}\sum_{k=1}^{m}\frac{n_k}{n}(\bm w^k_{t+1}-\bm w_t)$ to
the momentum and updates weight as $\bm w_{t+1}\leftarrow \bm w_t + \bm d_{t+1}$.
\end{itemize}

\begin{algorithm}[t]
 \small
 \caption{The algorithmic description of the FedNNNN protocol (Server).}
 \label{alg:Fed4N}
 \begin{algorithmic}[1]
 \REQUIRE
 \STATE initialize $\bm w_0, \bm d_0\leftarrow 0$
 \FOR{each round $t=0,1,2,...$}
 \STATE $m \leftarrow \max(C\cdot K,1)$
 \STATE $S_t \leftarrow$ (random set of $m$ clients)
 \FOR{each client $k\in S_t$ (in parallel)}
 \STATE send $\bm w_t$ to the clients
 \STATE $\bm w_{t+1}^k \leftarrow$ \textbf{ClientsUpdate($k,\bm w_t$)}
 \ENDFOR
 \STATE $N_{t+1}\leftarrow \norm*{\sum_{k=1}^{m}\frac{n_k}{n}(\bm w^k_{t+1}-\bm w_t)}$
 \STATE $E_{t+1}\leftarrow \sum_{k=1}^{m}\frac{n_k}{n}\norm*{\bm w^k_{t+1}-\bm w_t}$
 \STATE $\bm d_{t+1}\leftarrow \gamma \bm d_t + \beta \frac{E_{t+1}}{N_{t+1}}\sum_{k=1}^{m}\frac{n_k}{n}(\bm w^k_{t+1}-\bm w_t)$
 \STATE $\bm w_{t+1}\leftarrow \bm w_t + \bm d_{t+1}$
 \ENDFOR
 \end{algorithmic}
\end{algorithm}

\subsection*{Proof of Proposition~1}
In this section, we show a formal proof of Proposition 1. Before delving
into the proof of Proposition~1, we first outline an important lemma.
\begin{lemma}
Let $m$ be some integer. The following inequality holds
\begin{align}
    \norm*{\sum_{k=1}^{m} \alpha_{k}\Delta \bm w^k_{t+1}}\leq  \sum_{k=1}^m \alpha_{k}\norm*{\Delta \bm w^k_{t+1}}
\end{align}
for all $t\in(0, 1, \cdots)$ and any real numbers $\alpha_{1}, \alpha_{2}, \cdots, \alpha_{m}$.
\end{lemma}

\begin{proof}
  We prove Lemma~1 through an induction on $m$.\\
  {\bf{Base case}}: Let $m=1$. Then, Equation~(1) becomes
  \begin{align}
    \norm*{\alpha_{1}\Delta \bm w^1_{t+1}}=\norm*{\alpha_{1}\Delta \bm w^1_{t+1}},
  \end{align}
  and the equality obviously holds for any $t=(0, 1, \cdots)$.\\
  
  {\bf{Inductive case}}:
  Assume that Equation~(1) holds for $m$. For $m+1$, the LHS of Equation~(1) becomes
  \begin{align}
    &\norm*{\sum_{k=1}^{m+1} \alpha_{k}\Delta \bm w^k_{t+1}} =\norm*{\sum_{k=1}^{m} \alpha_{k}\Delta \bm w^k_{t+1} + \alpha_{m+1}\Delta\bm w^{m+1}_{t+1}}.&
\end{align}
  
  The Pythagorean
  inequality states that, for any two real vectors ${\bf{x}}, {\bf{y}}$,
  the following inequality holds.
  \begin{align}
    \norm{\bf{x}+\bf{y}}\leq \norm{\bf{x}}+\norm{\bf{y}}.
  \end{align}
  Then, we can derive the following inequality
  \begin{align}
    \norm*{\sum_{k=1}^{m} \alpha_{k}\Delta \bm w^k_{t+1} + \alpha_{m+1}\Delta\bm w^{m+1}_{t+1}}
    \leq \norm*{\sum_{k=1}^{m} \alpha_{k}\Delta \bm w^k_{t+1}} + \norm*{\alpha_{m+1}\Delta\bm w^{m+1}_{t+1}}
  \end{align}
  Since $\norm*{\sum_{k=1}^{m} \alpha_{k}\Delta \bm w^k_{t+1}}\leq \sum_{k=1}^m \alpha_{k}\norm*{\Delta \bm w^k_{t+1}}$
  as assumed, Equation~(5) becomes
  \begin{align}
    \norm*{\sum_{k=1}^{m} \alpha_{k}\Delta \bm w^k_{t+1} + \alpha_{m+1}\Delta\bm w^{m+1}_{t+1}}
    \leq \sum_{k=1}^m \alpha_{k}\norm*{\Delta \bm w^k_{t+1}} + \norm*{\alpha_{m+1}\Delta\bm w^{m+1}_{t+1}}
  \end{align}
  Since it is also obvious that $\norm*{\alpha_{m+1}\Delta\bm w^{m+1}_{t+1}}=\alpha_{m+1}\norm*{\Delta\bm w^{m+1}_{t+1}}$,
  \begin{eqnarray}
      \sum_{k=1}^m \alpha_{k}\norm*{\Delta \bm w^k_{t+1}} + \norm*{\alpha_{m+1}\Delta\bm w^{m+1}_{t+1}}
      &=&\sum_{k=1}^m \alpha_{k}\norm*{\Delta \bm w^k_{t+1}} +\alpha_{m+1}\norm*{\Delta\bm w^{m+1}_{t+1}} \nonumber\\
      &=&\sum_{k=1}^{m+1} \alpha_{k}\norm*{\Delta \bm w^k_{t+1}}
  \end{eqnarray}
  Consequently, we know that
  \begin{align}
      \norm*{\sum_{k=1}^{m+1} \alpha_{k}\Delta \bm w^k_{t+1}}\leq \sum_{k=1}^{m+1} \alpha_{k}\norm*{\Delta \bm w^k_{t+1}},
  \end{align}
  and the lemma follows.
  \qed
\end{proof}

\begin{proposition}
\label{prop:inequality2}
The following inequality holds
\begin{equation}
    N_{t+1}\leq E_{t+1}
\end{equation}
\text{for all $t\in (0, 1, \cdots)$}
\end{proposition}
\begin{proof}
Since $N_{t+1}=\norm*{\sum_{k=1}^{m} \frac{n_{k}}{n}\Delta \bm w^k_{t+1}}$ and 
$E_{t+1}=\sum_{k=1}^m \frac{n_{k}}{n}\norm*{\Delta \bm w^k_{t+1}}$ as defined, let $\alpha_{k}=\frac{n_{k}}{n}$.
Then, by Lemma~1, we know that 
\begin{align}
N_{t+1}=\norm*{\sum_{k=1}^{m} \alpha_{k}\Delta \bm w^k_{t+1}}\leq \sum_{k=1}^m \alpha_{k}\norm*{\Delta \bm w^k_{t+1}}=E_{t+1},
\end{align}
and the proposition follows.
\qed
\end{proof}


\subsection*{Experiments}
\subsubsection*{Models}

Table~\ref{tab:cnn_MN} and Table~\ref{tab:cnn_C10} show 
the models we used for the experiments of MNIST and CIFAR10, 
respectively.  We adopted the cross entropy loss function 
in all models.

\begin{table}[t!]
    \footnotesize
    \caption{The Neural Architecture Utilized in The MNIST Experiment}
    \begin{center}
        \begin{tabular}{l|c|c|c|c}\hline
            Layer & Input $(c,w,h)$& Output $(c,w,h)$& Kernel & Stride \\\hline\hline
            Input&-&$(1,28,28)$&-&-  \\\hline
            Conv1&$(1,28,28)$&$(20,24,24)$&5&1 \\\hline
            ReLU&-&-&-&- \\\hline
            Maxpool&$(20,24,24)$&$(20,12,12)$&2&2  \\\hline
            Conv2&$(20,12,12)$&$(50,8,8)$&5&1 \\\hline
            ReLU&-&-&-&- \\\hline
            Maxpool&$(50,8,8)$&$(50,4,4)$&2&2  \\\hline
            Fc1&$50*4*4$&$500$&-&-\\\hline
            ReLU&-&-&-&-  \\\hline
            Fc2&$500$&$10$&-&-\\\hline
        \end{tabular}
    \end{center}
    \label{tab:cnn_MN}
\end{table}

\begin{table}[t!]
    \footnotesize
    \caption{The Neural Architecture Utilized in The CIFAR10 Experiment}
    \begin{center}
        \begin{tabular}{l|c|c|c|c}\hline
            Layer & Input $(c,w,h)$& Output $(c,w,h)$& Kernel & Stride \\\hline\hline
            Input&-&$(3,32,32)$&-&-  \\\hline
            Conv11&$(3,32,32)$&$(32,32,32)$&3&1 \\\hline
            BN+ReLU&-&-&-&- \\\hline
            Conv12&$(32,32,32)$&$(32,32,32)$&3&1 \\\hline
            BN+ReLU&-&-&-&- \\\hline
            Maxpool&$(32,32,32)$&$(32,16,16)$&2&2  \\\hline
            Conv21&$(32,16,16)$&$(64,16,16)$&3&1 \\\hline
            BN+ReLU&-&-&-&- \\\hline
            Conv22&$(64,16,16)$&$(64,16,16)$&3&1 \\\hline
            BN+ReLU&-&-&-&- \\\hline
            Maxpool&$(64,16,16)$&$(64,8,8)$&2&2  \\\hline
            Conv31&$(64,8,8)$&$(128,8,8)$&3&1 \\\hline
            BN+ReLU&-&-&-&- \\\hline
            Conv32&$(128,8,8)$&$(128,8,8)$&3&1 \\\hline
            BN+ReLU&-&-&-&- \\\hline
            Maxpool&$(128,8,8)$&$(128,4,4)$&2&2  \\\hline
            Fc1&$128*4*4$&$382$&-&-\\\hline
            ReLU&-&-&-&-  \\\hline
            Fc2&$382$&$192$&-&-\\\hline
            ReLU&-&-&-&-  \\\hline
            Fc3&$192$&$10$&-&-\\\hline
        \end{tabular}
    \end{center}
    \label{tab:cnn_C10}
\end{table}

\subsubsection*{Determination of Hyperparameters}

In this section, we explain how we determined the parameters 
used in the experiment.
The following heuristic procedure is adopted on deciding hyperparameters.
\begin{description}
    \item[Step 1] First of all, divide the training dataset into two datasets.
    One dataset is used to train the neural network model, 
    while the other
    dataset is used to adjust the hyperparameters, referred to as the validation
    dataset.
    \item[Step 2] Decide the learning rate $\eta$. 
    We performed several epochs (30 epochs on CIFAR10 and 50 epochs on MNIST) for the
    training of non-IID and B conditions with $\eta$ of the candidates, 
    and the one with the best test accuracy was used in all methods and conditions.
    In MNIST, we choose from range [0.02,0.15] by every 0.01, and in CIFAR10 we
    choose from range [0.01,0.11] by every 0.02.
    
    \item[Step 3] While fixing $\eta$ at the best value in Step-2, with respect to each dataset and condition, we 
    varied $\mu$, $\beta$ and $\gamma$ independently and choose the one with 
    the best test accuracy.
\end{description}

Table~\ref{tab:params2} lists all the parameters used for each FL method, for each condition.
We choose $\mu$, $\beta$ and $\gamma$ as well as the learning rate $\eta$.
In FedNNNN, we varied $\beta$ and $\gamma$ independently and selected the the combination that gave the best accuracy.
%
%

\begin{table}[tb]
\footnotesize
\caption{Summary of Hyperparameters Used in the Experiment (scenario-wise)}
\label{tab:params2}
\begin{tabular}{lcccccccc}
\multirow{2}{*}{} & \multicolumn{2}{c}{IID-B} & \multicolumn{2}{c}{NonIID-B} & \multicolumn{2}{c}{IID-UB} & \multicolumn{2}{c}{NonIID-UB} \\ 
                           & MNIST     & CIFAR10    & MNIST     & CIFAR10    & MNIST     & CIFAR10    & MNIST     & CIFAR10    \\ \hline\hline
FedProx\\$\mu$             & 0.005     & 0.015      & 0.015     & 0.015      & 0.005     & 0.005       & 0.02      & 0.01       \\ \hline
Norm-Norm \\$\beta$        & 1.1       & 0.6        & 1.0       & 0.6        & 1.0       & 0.7        & 0.9       & 0.7        \\ \hline
Momentum \\$\gamma$        & 0.8       & 0.9        & 0.9       & 0.9        & 0.7       & 0.9        & 0.8       & 0.8        \\ \hline
FedNNNN \\$\beta$          & 0.6       & 0.7        & 0.7       & 0.6        & 0.7       & 0.8        & 0.7       & 0.7        \\ \hline
FedNNNN \\$\gamma$         & 0.7       & 0.8        & 0.8       & 0.7        & 0.7       & 0.8        & 0.8       & 0.6        \\ \hline
\end{tabular}
\end{table}
   
\subsubsection*{Experimental Results}
Here, we show accuracy curve and update norm on MNIST, CIFAR10.
Figures~\ref{fig:mnist_acc_iid-b} $\sim$ ~\ref{fig:mnist_norm_conv_niid-ub} show the
accuracy and update norm on MNIST.
Figure~\ref{fig:cifar10_acc_iid-b}  $\sim$~\ref{fig:cifar10_norm_conv_niid-ub} show those
 on CIFAR10 is shown.

\begin{figure}[t]
    \begin{tabular}{ccc}
        \begin{minipage}[t]{0.32\linewidth}
        \centering
        \includegraphics[width=\linewidth]{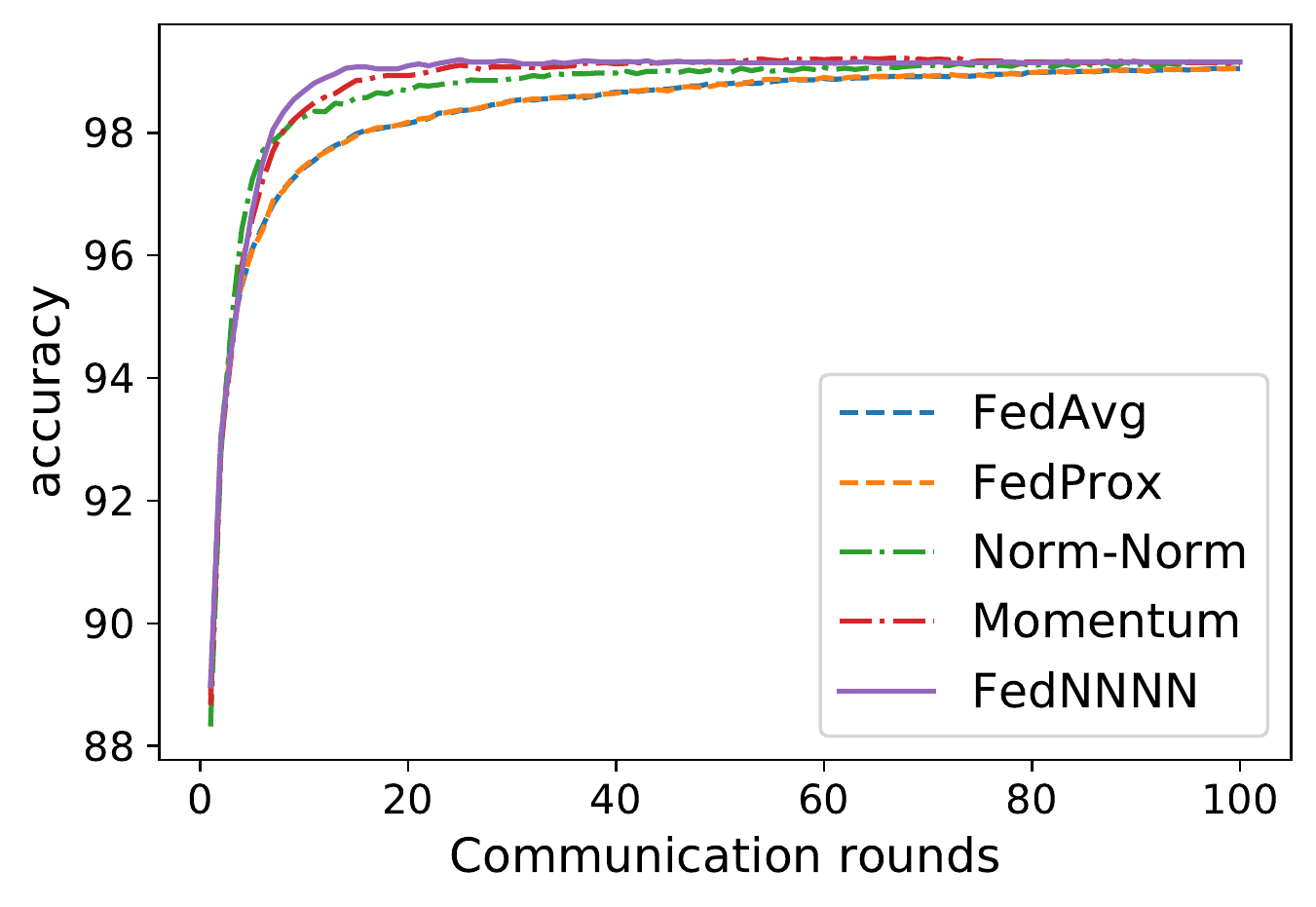}
        \caption{Accuracy curves on the MNIST dataset under IID - B condition.}
        \label{fig:mnist_acc_iid-b}       
      \end{minipage}
      \hspace{0.1cm}
      \begin{minipage}[t]{0.32\linewidth}
        \centering
        \includegraphics[width=\linewidth]{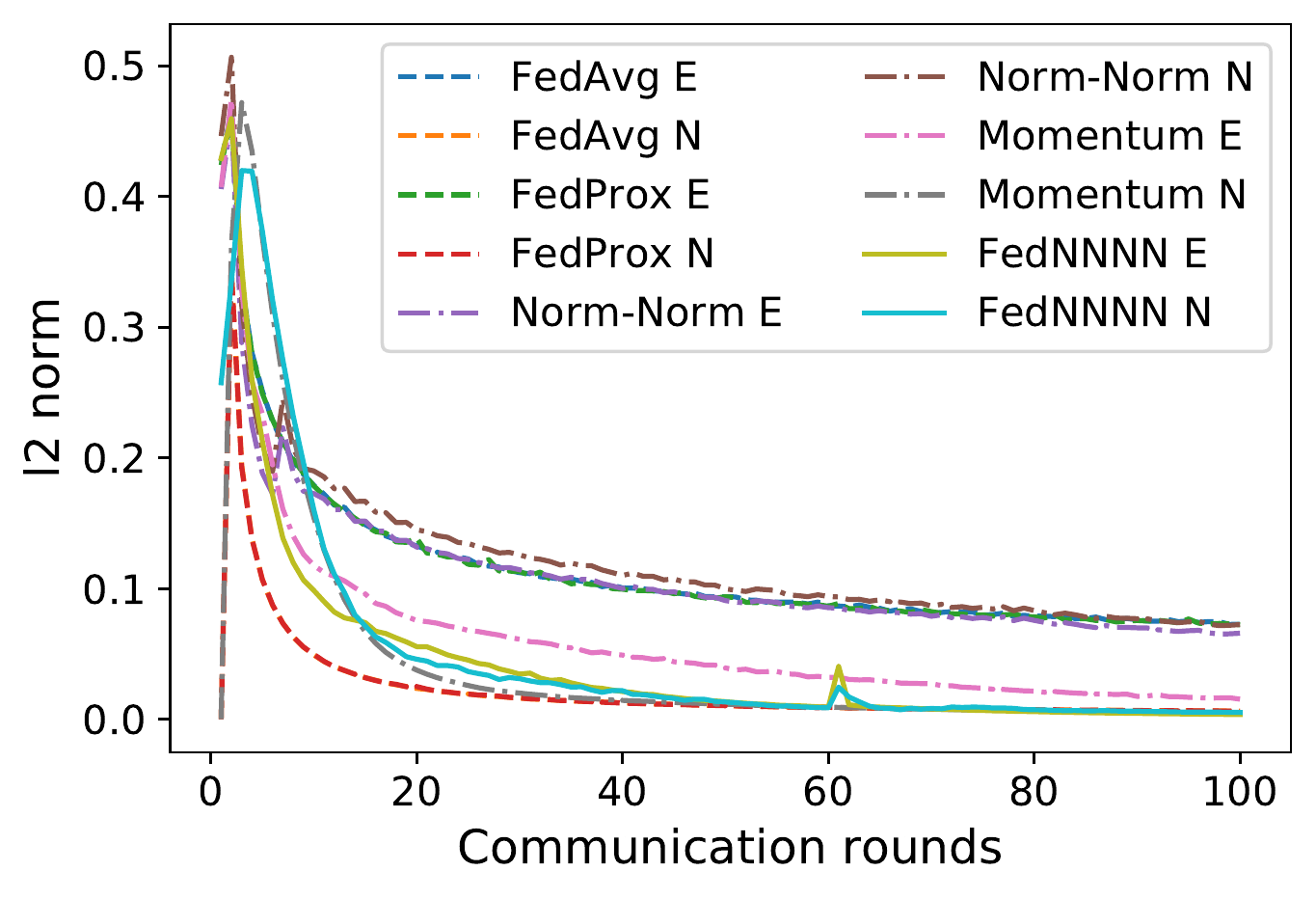}
        \caption{The $N$ and $E$ of fully connected layer calculated on the 
        MNIST dataset under IID - B condition.}
        \label{fig:mnist_norm_fc_iid-b}       
      \end{minipage} 
      \hspace{0.1cm}
      \begin{minipage}[t]{0.32\linewidth}
        \centering
        \includegraphics[width=\linewidth]{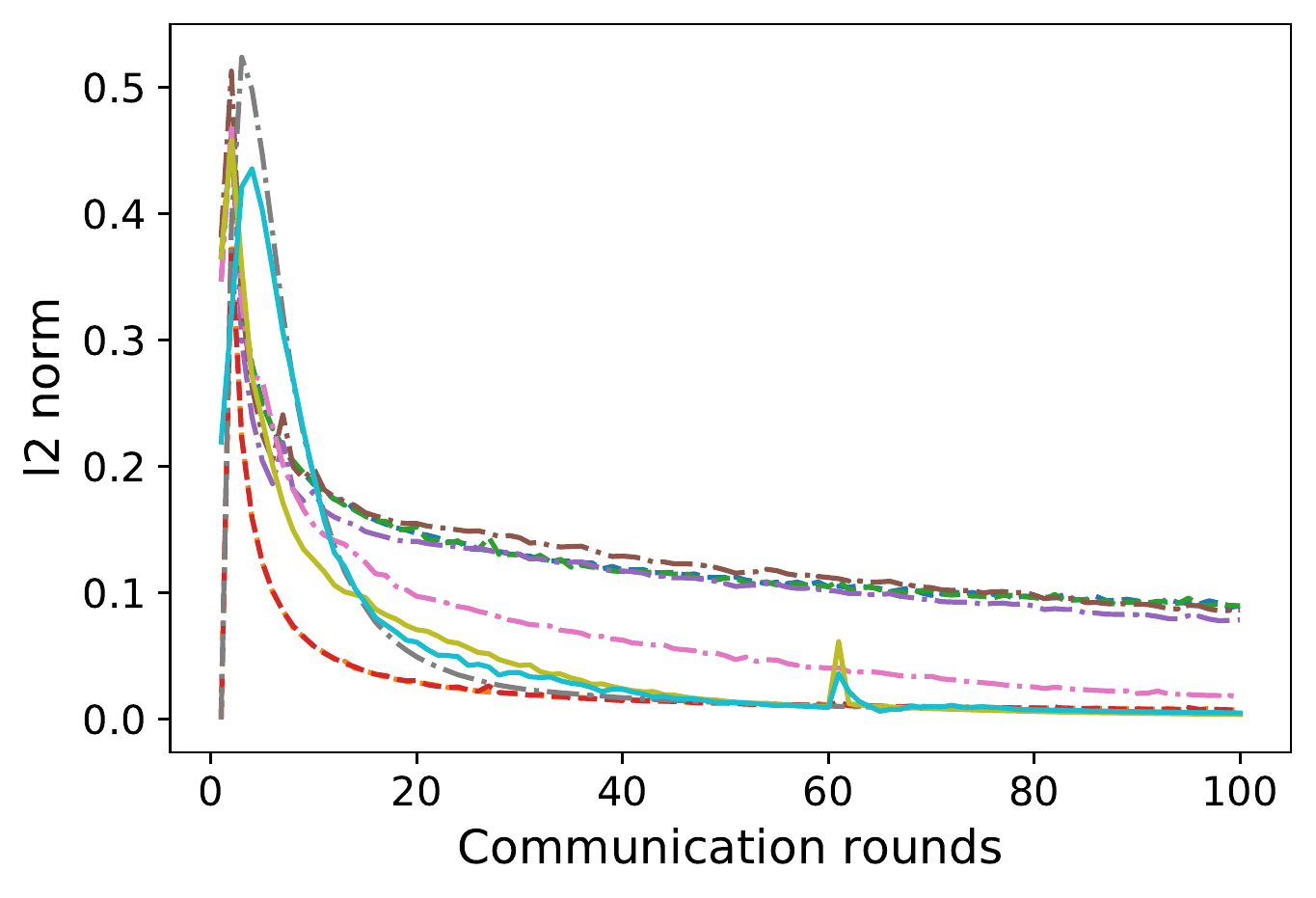}
        \caption{The $N$ and $E$ of convolutional layer calculated on the MNIST
        dataset under IID - B condition.}
        \label{fig:mnist_norm_conv_iid-b}       
      \end{minipage} \\

      \begin{minipage}[t]{0.32\linewidth}
        \centering
        \includegraphics[width=\linewidth]{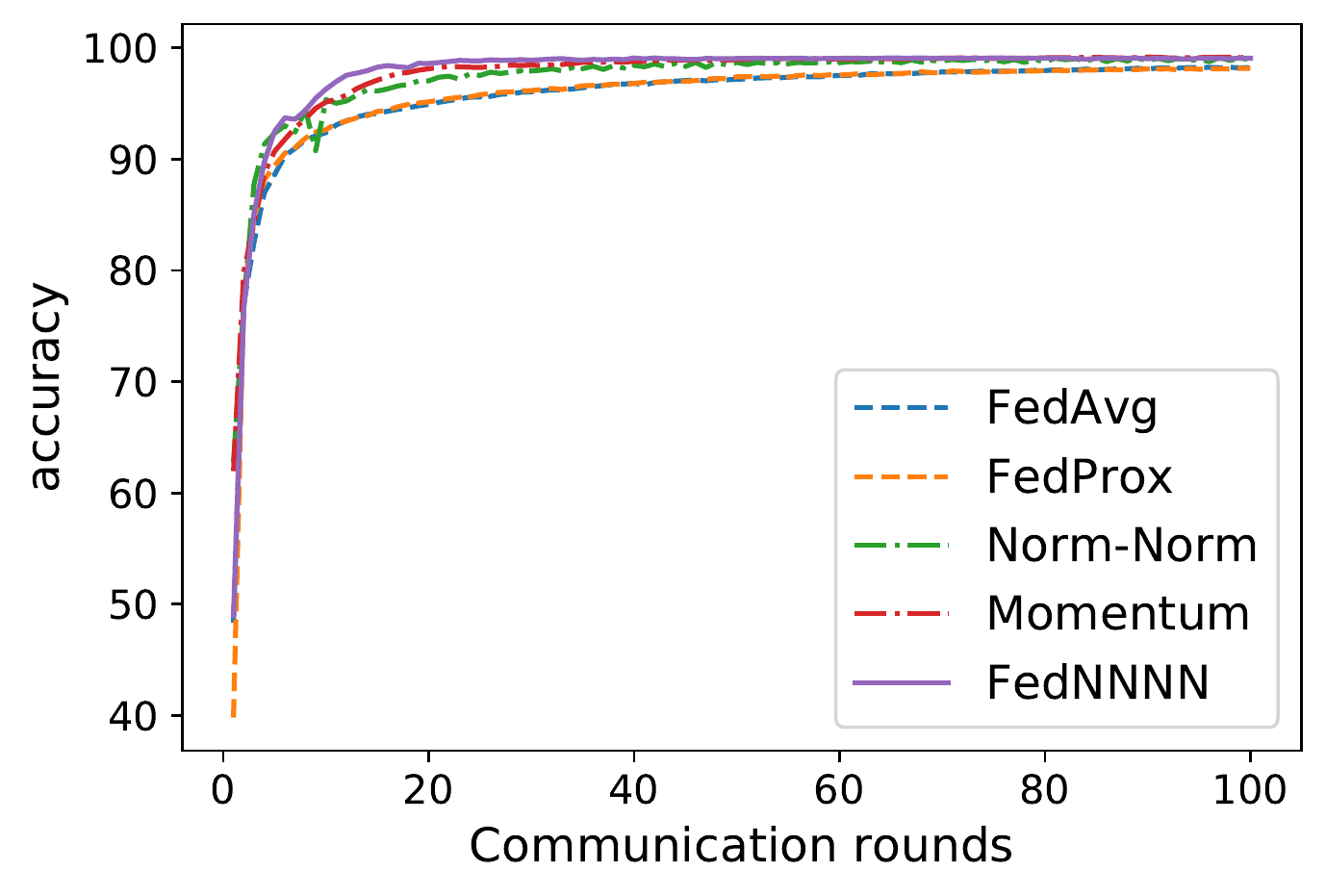}
        \caption{Accuracy curves on the MNIST dataset under non-IID - B condition.}
        \label{fig:mnist_acc_niid-b}       
      \end{minipage}
      \hspace{0.1cm}
      \begin{minipage}[t]{0.32\linewidth}
        \centering
        \includegraphics[width=\linewidth]{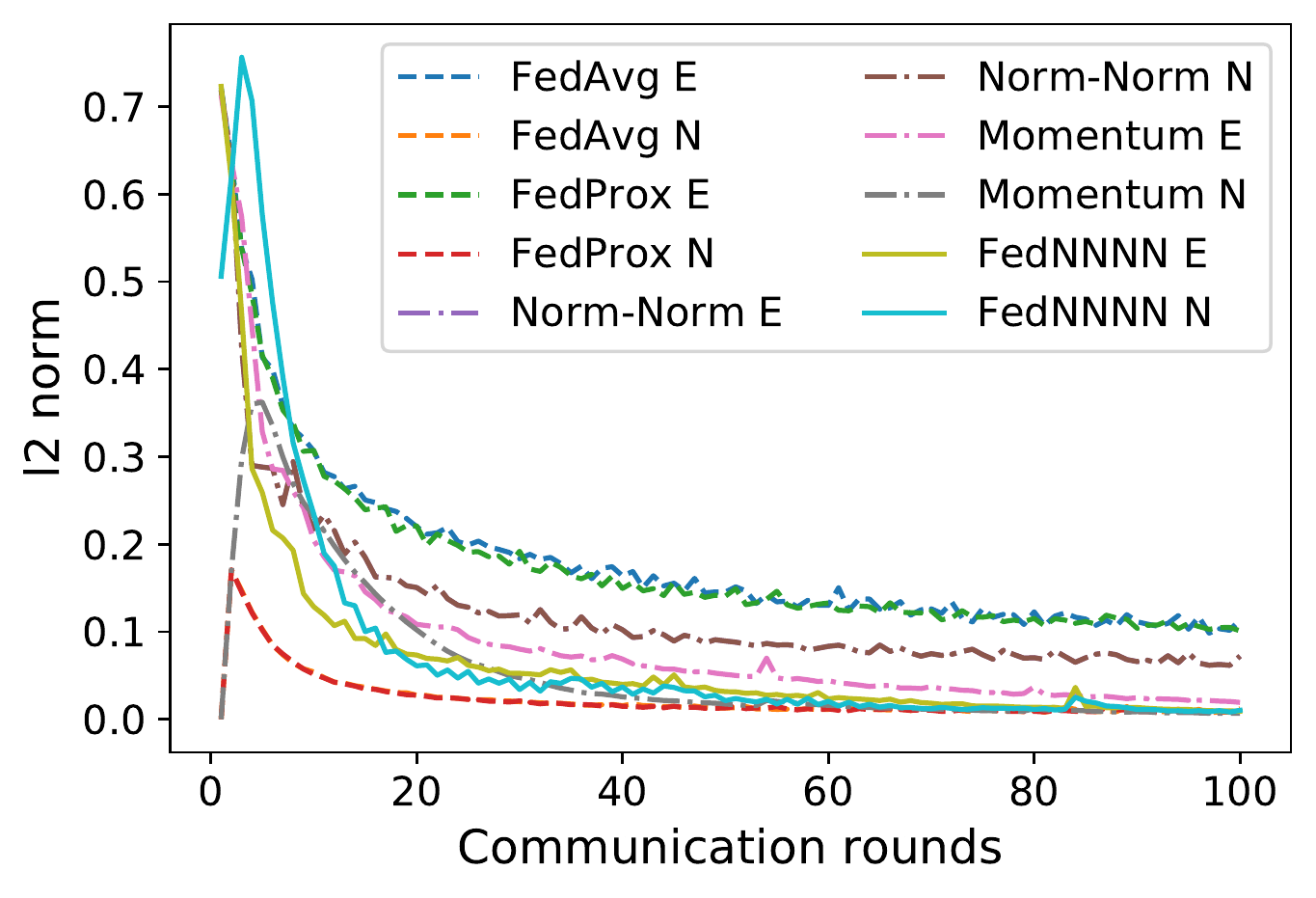}
        \caption{The $N$ and $E$ of fully connected layer calculated on the 
        MNIST dataset under non-IID - B condition.}
        \label{fig:mnist_norm_fc_niid-b}       
      \end{minipage} 
      \hspace{0.1cm}
      \begin{minipage}[t]{0.32\linewidth}
        \centering
        \includegraphics[width=\linewidth]{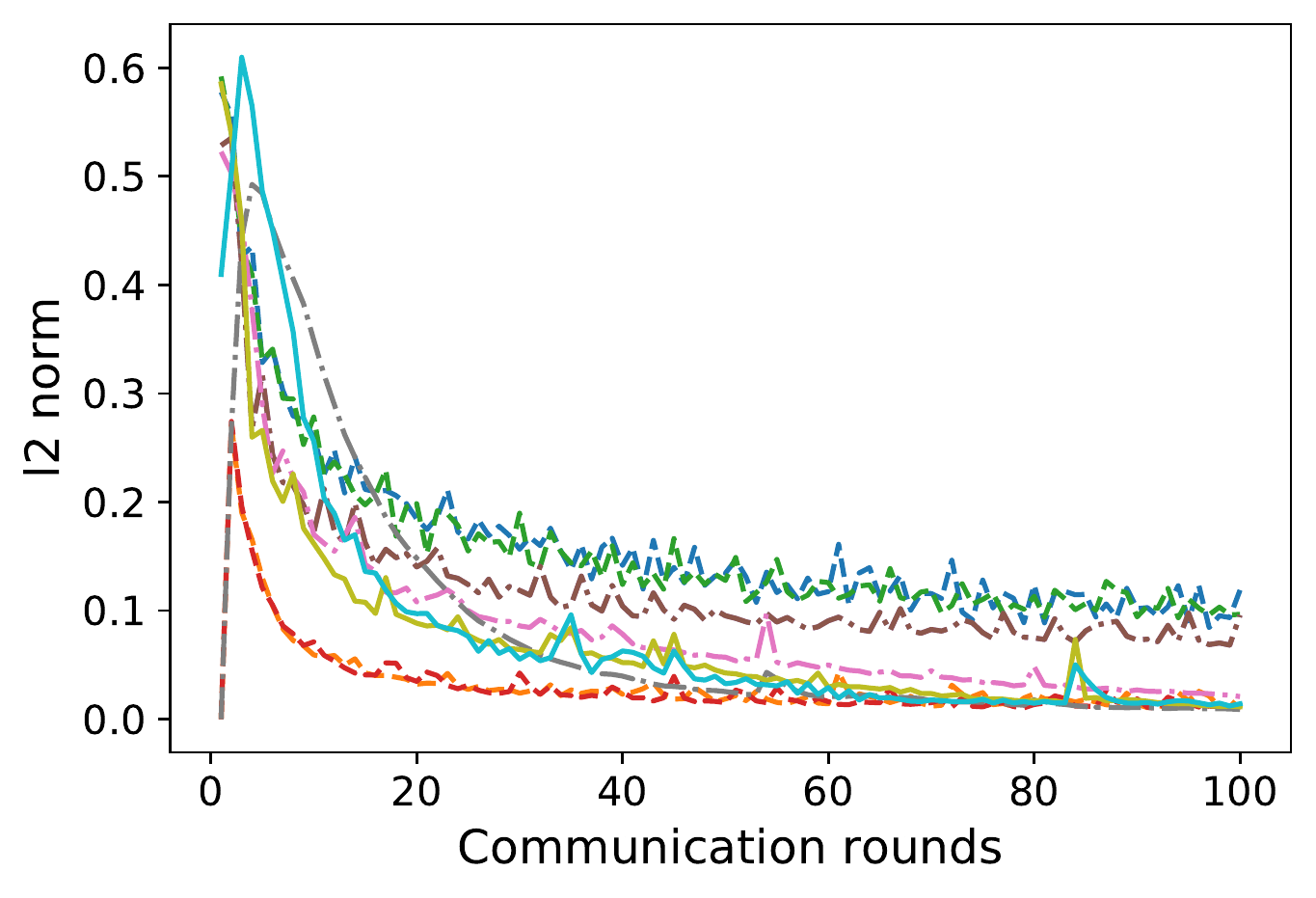}
        \caption{The $N$ and $E$ of convolutional layer calculated on the MNIST
        dataset under non-IID - B condition.}
        \label{fig:mnist_norm_conv}       
      \end{minipage}\\
%
        \begin{minipage}[t]{0.32\linewidth}
        \centering
        \includegraphics[width=\linewidth]{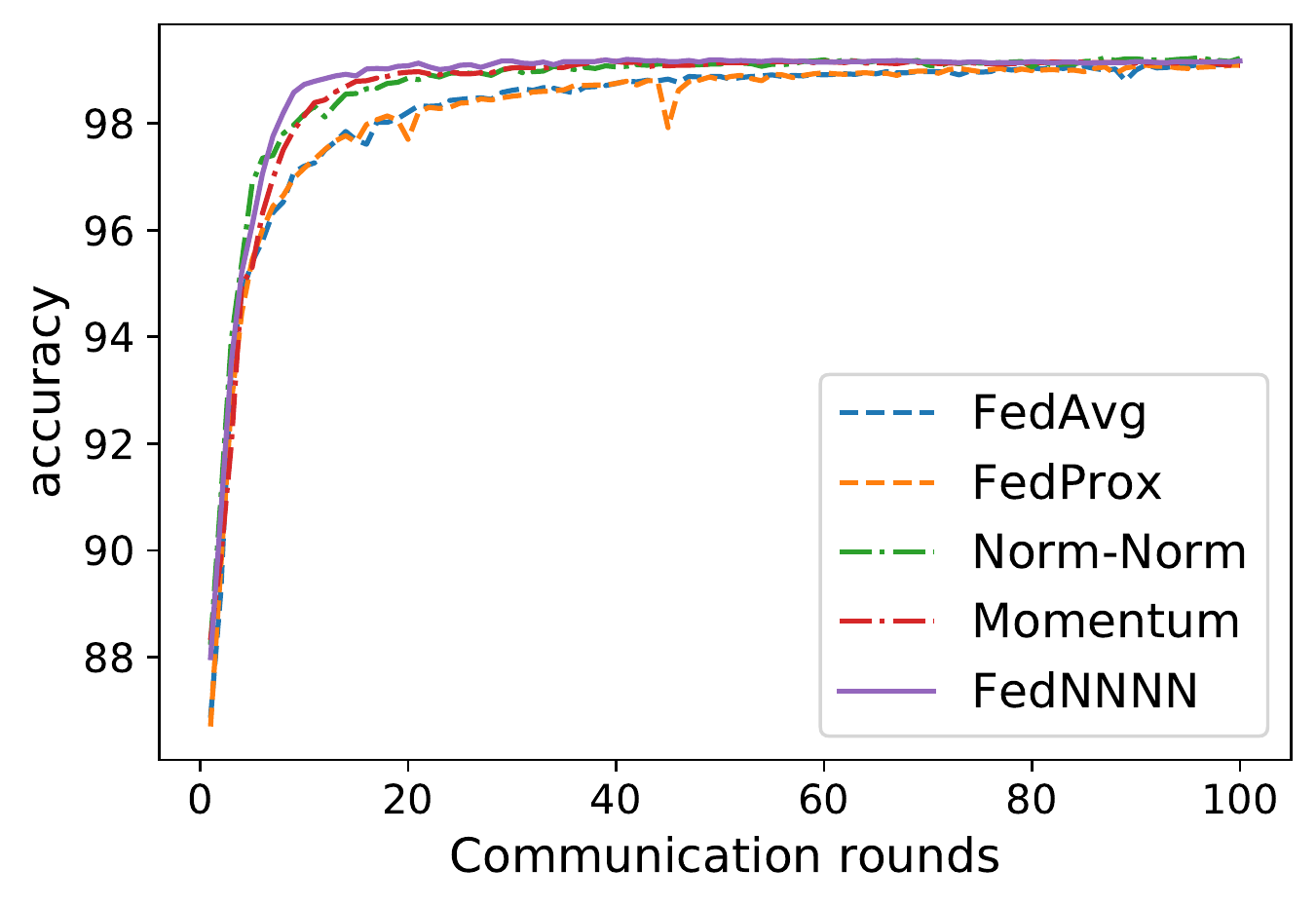}
        \caption{Accuracy curves on the MNIST dataset under IID - UB condition.}
        \label{fig:mnist_acc_iid-ub}       
      \end{minipage}
      \hspace{0.1cm}
      \begin{minipage}[t]{0.32\linewidth}
        \centering
        \includegraphics[width=\linewidth]{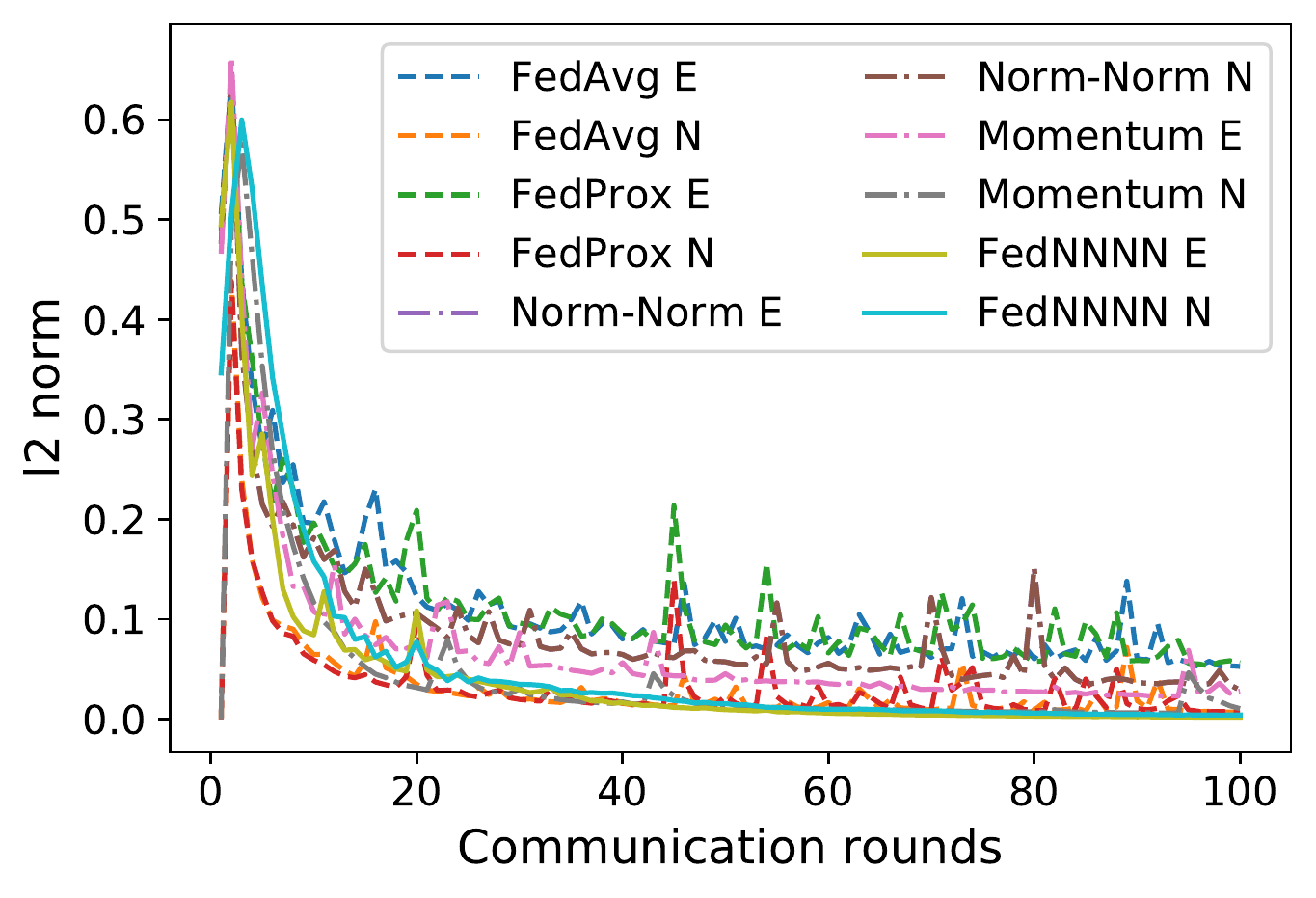}
        \caption{The $N$ and $E$ of fully connected layer calculated on the MNIST
        dataset under IID - UB condition.}
        \label{fig:mnist_norm_fc_iid-ub}       
      \end{minipage} 
      \hspace{0.1cm}
      \begin{minipage}[t]{0.32\linewidth}
        \centering
        \includegraphics[width=\linewidth]{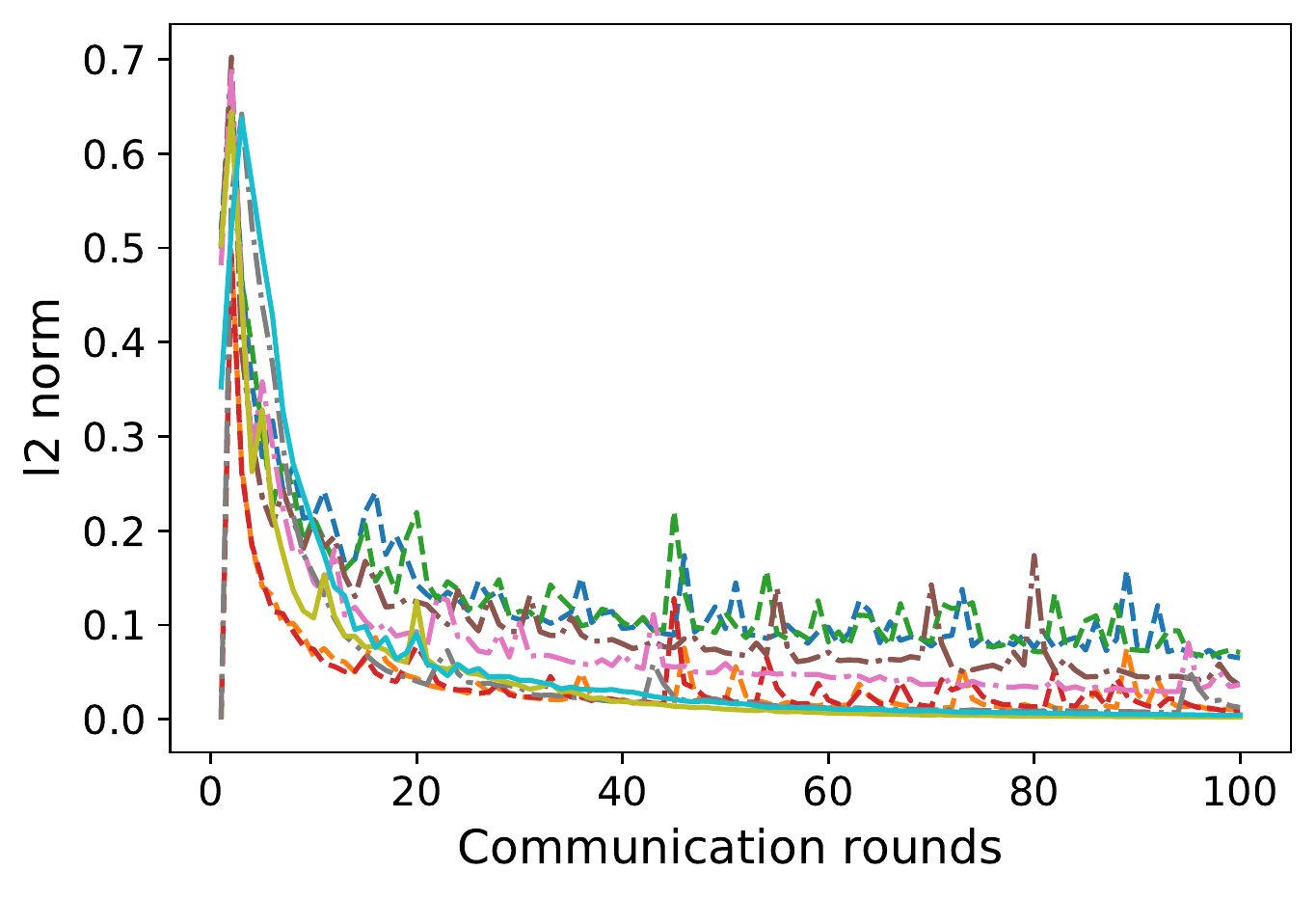}
        \caption{The $N$ and $E$ of convolutional layer calculated on the MNIST
        dataset under IID - UB condition.}
        \label{fig:mnist_norm_conv_iid-ub}       
      \end{minipage}\\
%
        \begin{minipage}[t]{0.32\linewidth}
        \centering
        \includegraphics[width=\linewidth]{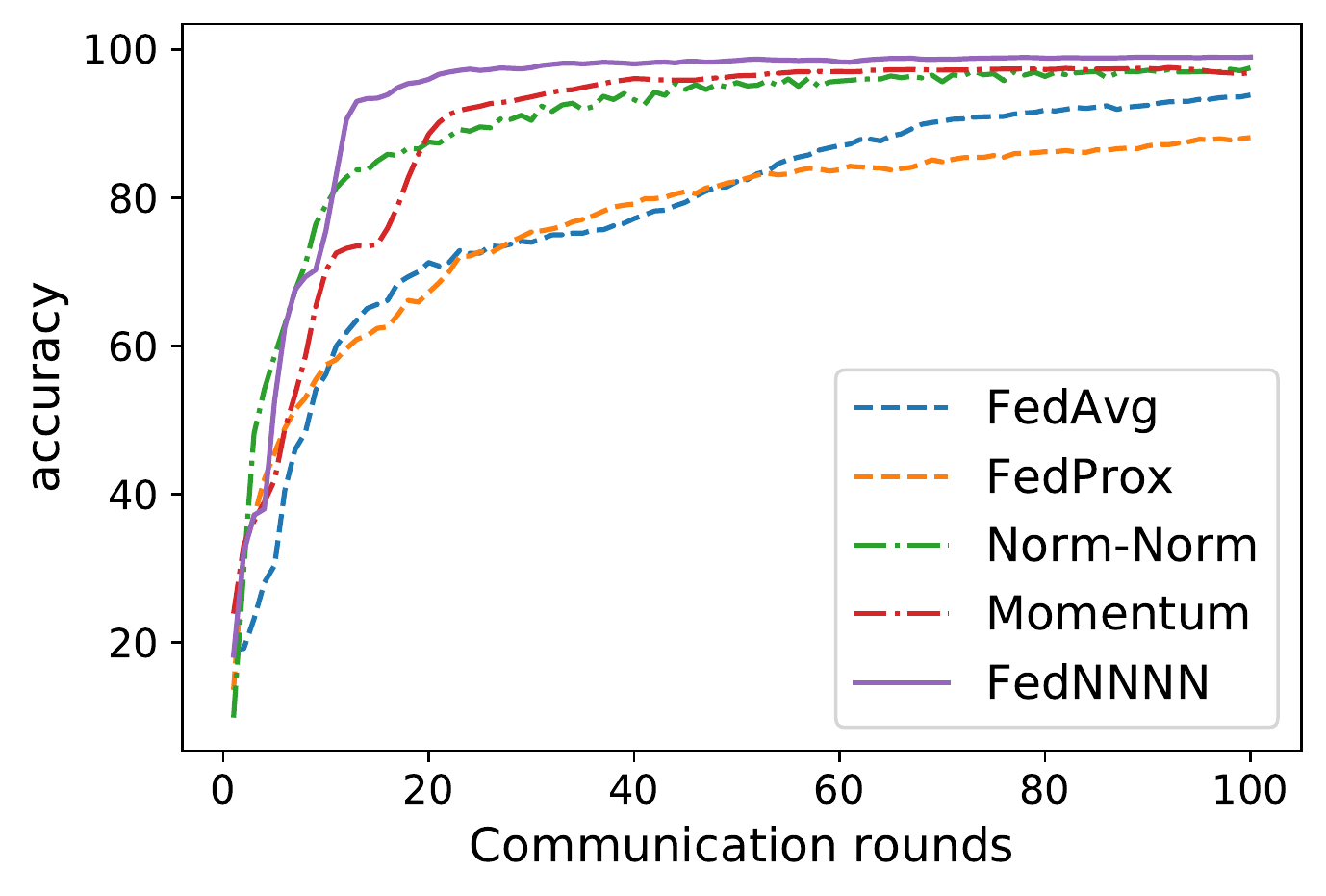}
        \caption{Accuracy curves on the MNIST dataset under non-IID - UB condition.}
        \label{fig:mnist_acc_niid-ub}       
      \end{minipage}
      \hspace{0.1cm}
      \begin{minipage}[t]{0.32\linewidth}
        \centering
        \includegraphics[width=\linewidth]{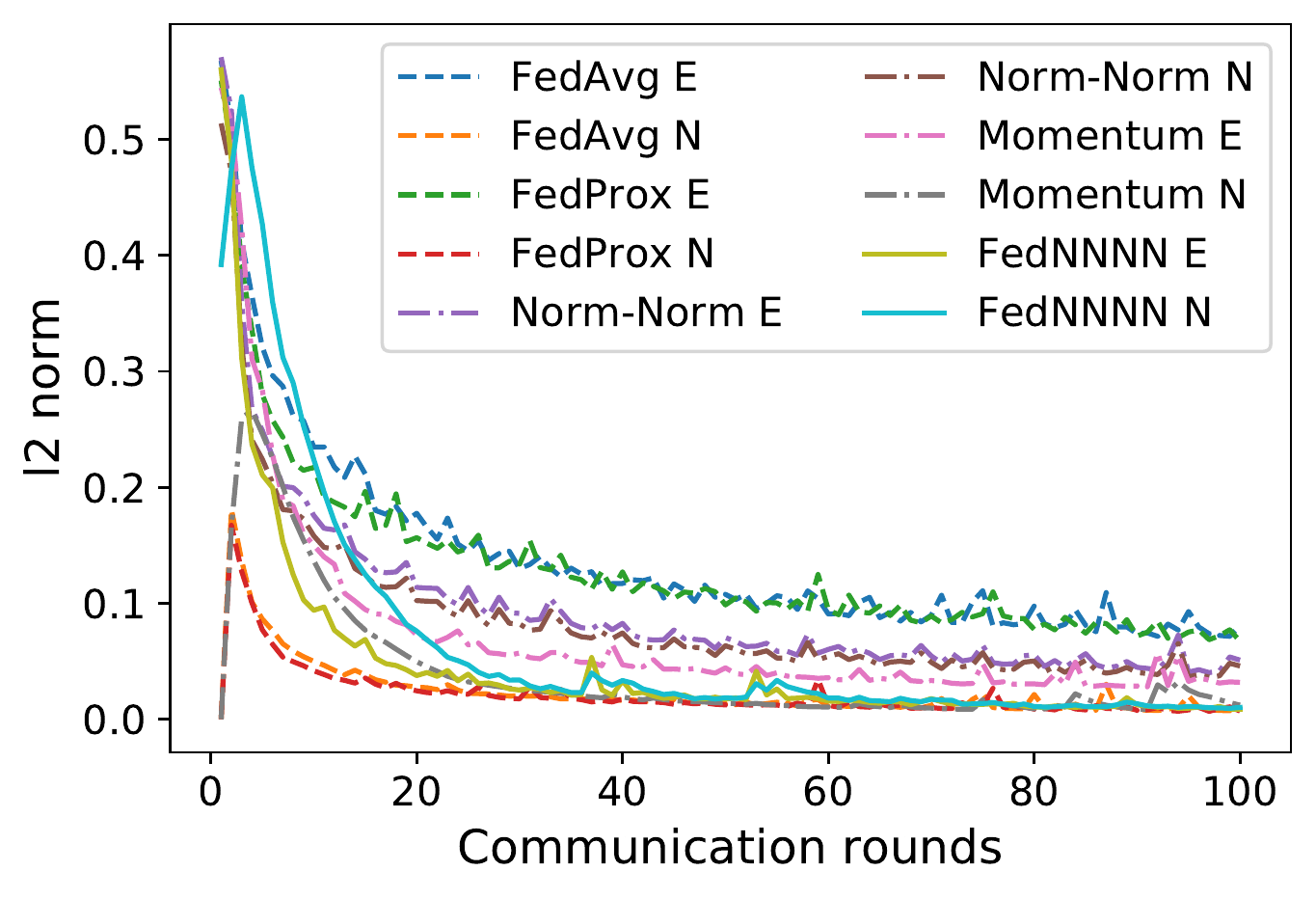}
        \caption{The $N$ and $E$ of fully connected layer calculated on the MNIST
        dataset under non-IID - UB condition.}
        \label{fig:mnist_norm_fc_niid-ub}       
      \end{minipage} 
      \hspace{0.1cm}
      \begin{minipage}[t]{0.32\linewidth}
        \centering
        \includegraphics[width=\linewidth]{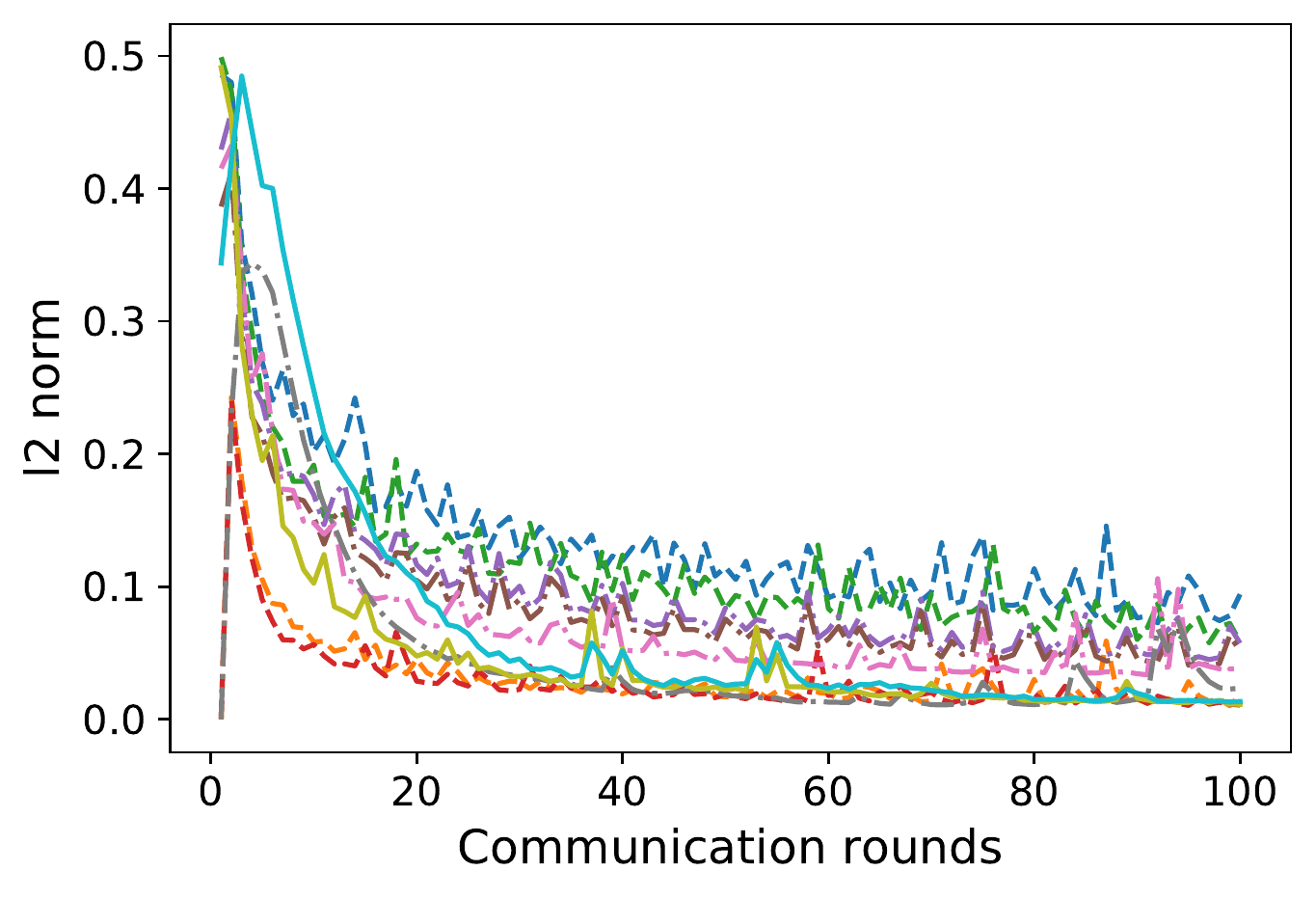}
        \caption{The $N$ and $E$ of convolutional layer calculated on the MNIST
        dataset under non-IID - UB condition.}
        \label{fig:mnist_norm_conv_niid-ub}       
      \end{minipage}
    \end{tabular}
\end{figure}

\begin{figure}[t]
    \begin{tabular}{ccc}
        \begin{minipage}[t]{0.32\linewidth}
        \centering
        \includegraphics[width=\linewidth]{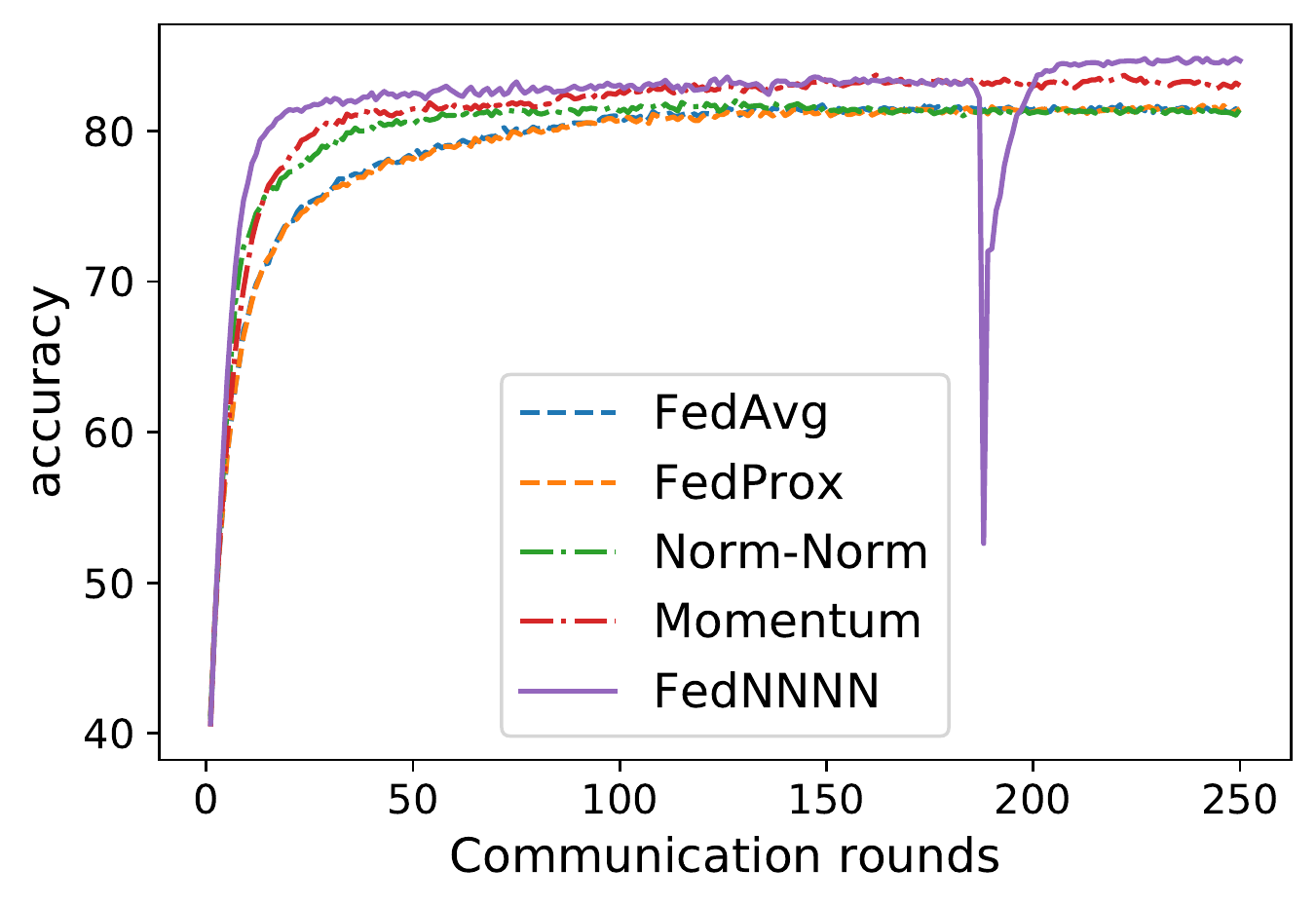}
        \caption{Accuracy curves on the CIFAR10 dataset under IID - B condition.}
        \label{fig:cifar10_acc_iid-b}       
      \end{minipage}
      \hspace{0.1cm}
      \begin{minipage}[t]{0.32\linewidth}
        \centering
        \includegraphics[width=\linewidth]{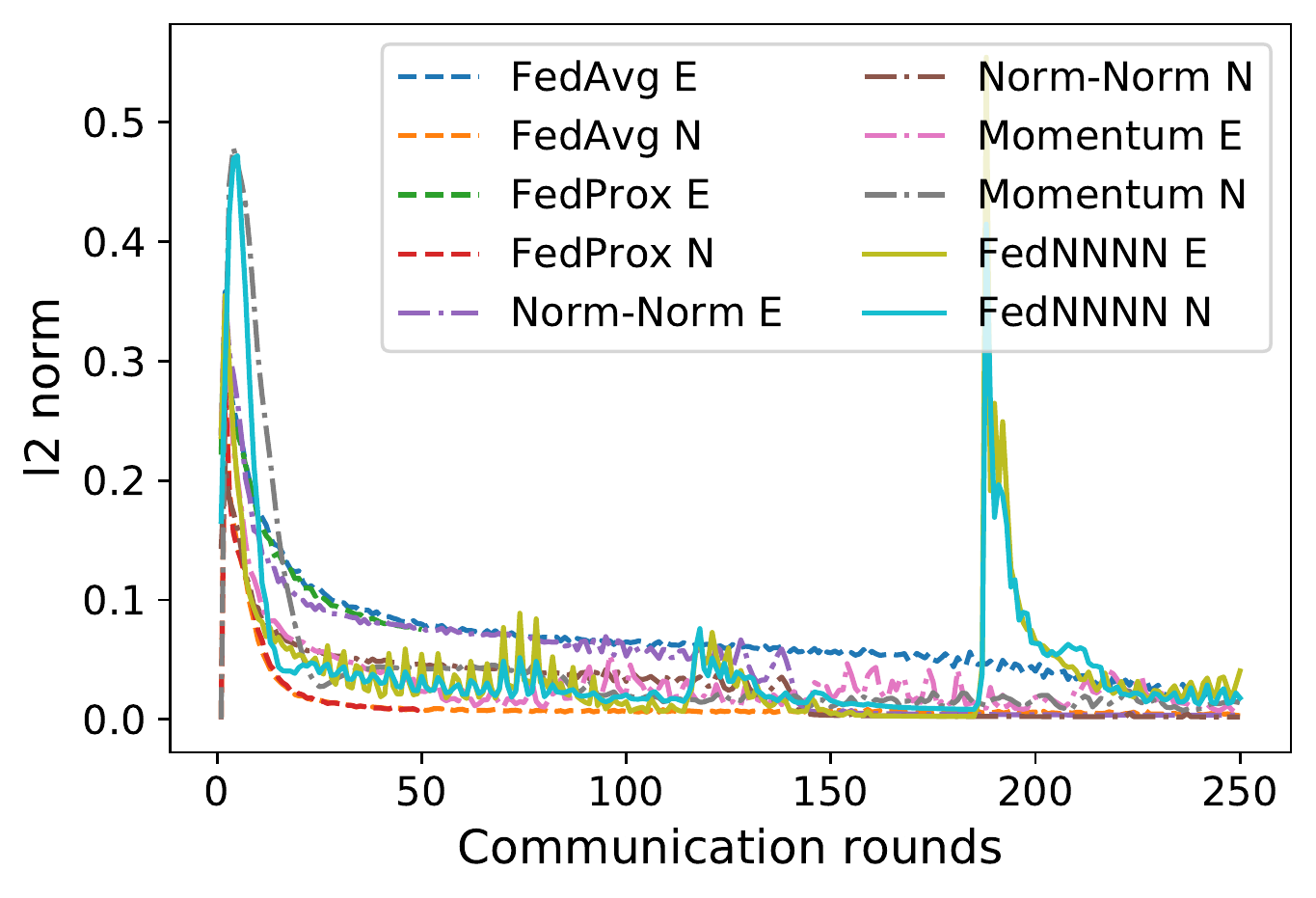}
        \caption{The $N$ and $E$ of fully connected layer calculated on the CIFAR10 dataset
        under IID - B condition.}
        \label{fig:cifar10_norm_fc_iid-b}       
      \end{minipage} 
      \hspace{0.1cm}
      \begin{minipage}[t]{0.32\linewidth}
        \centering
        \includegraphics[width=\linewidth]{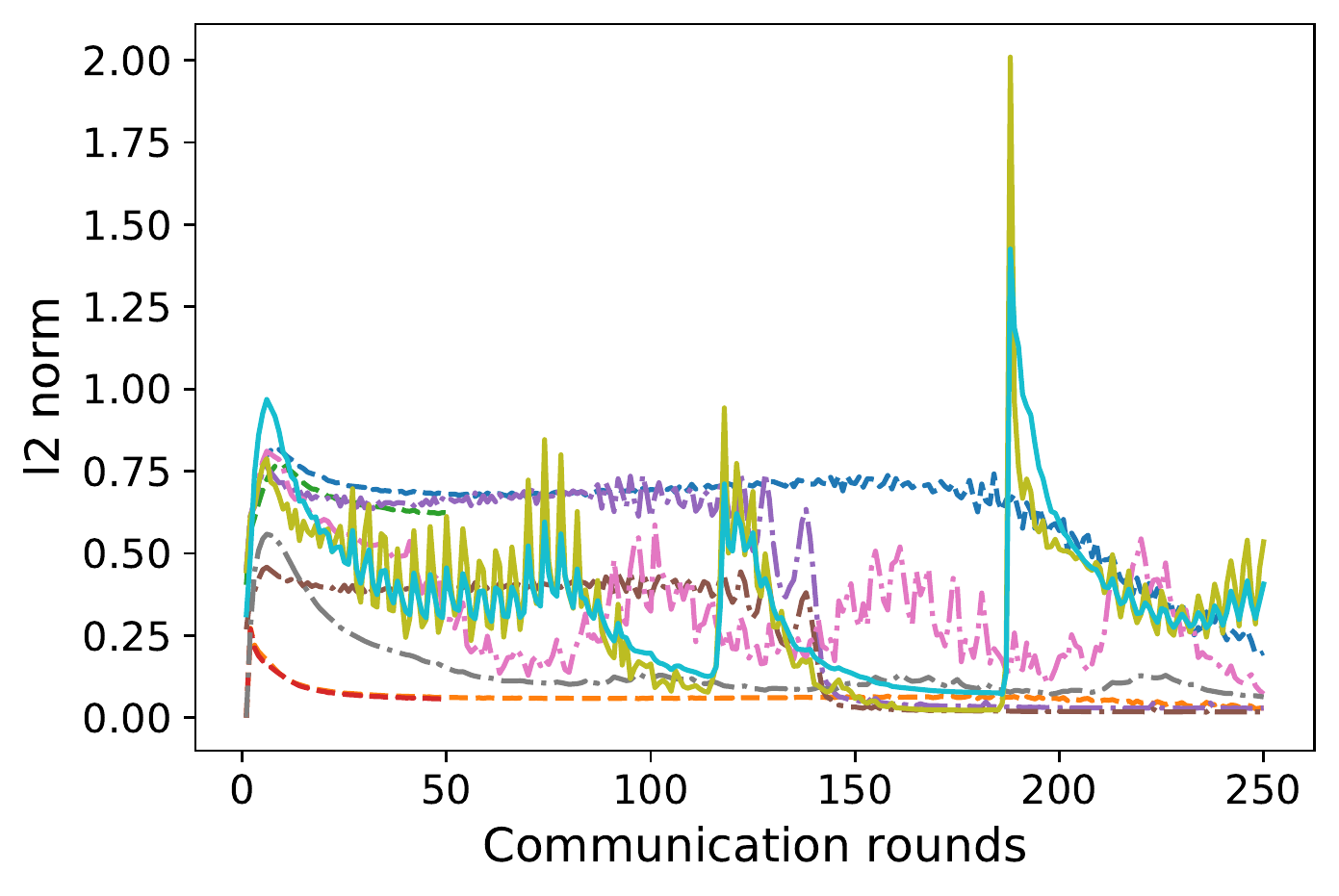}
        \caption{The $N$ and $E$ of convolutional layer calculated on the CIFAR10 dataset
        under IID - B condition.}
        \label{fig:cifar10_norm_conv_iid-b}       
      \end{minipage} \\
%
        \begin{minipage}[t]{0.32\linewidth}
        \centering
        \includegraphics[width=\linewidth]{figure/accuracy_comp.pdf}
        \caption{Accuracy curves on the CIFAR10 dataset under non-IID - B condition.}
        \label{fig:cifar10_acc2}       
      \end{minipage}
      \hspace{0.1cm}
      \begin{minipage}[t]{0.32\linewidth}
        \centering
        \includegraphics[width=\linewidth]{figure/norm_comp_fc3.pdf}
        \caption{The $N$ and $E$ of fully connected layer calculated on the CIFAR10 dataset
        under non-IID - B condition.}
        \label{fig:cifar10_norm_fc2}       
      \end{minipage} 
      \hspace{0.1cm}
      \begin{minipage}[t]{0.32\linewidth}
        \centering
        \includegraphics[width=\linewidth]{figure/norm_comp_conv32.pdf}
        \caption{The $N$ and $E$ of convolutional layer calculated on the CIFAR10 dataset
        under non-IID - B condition.}
        \label{fig:cifar10_norm_conv2}       
      \end{minipage}\\
%
        \begin{minipage}[t]{0.32\linewidth}
        \centering
        \includegraphics[width=\linewidth]{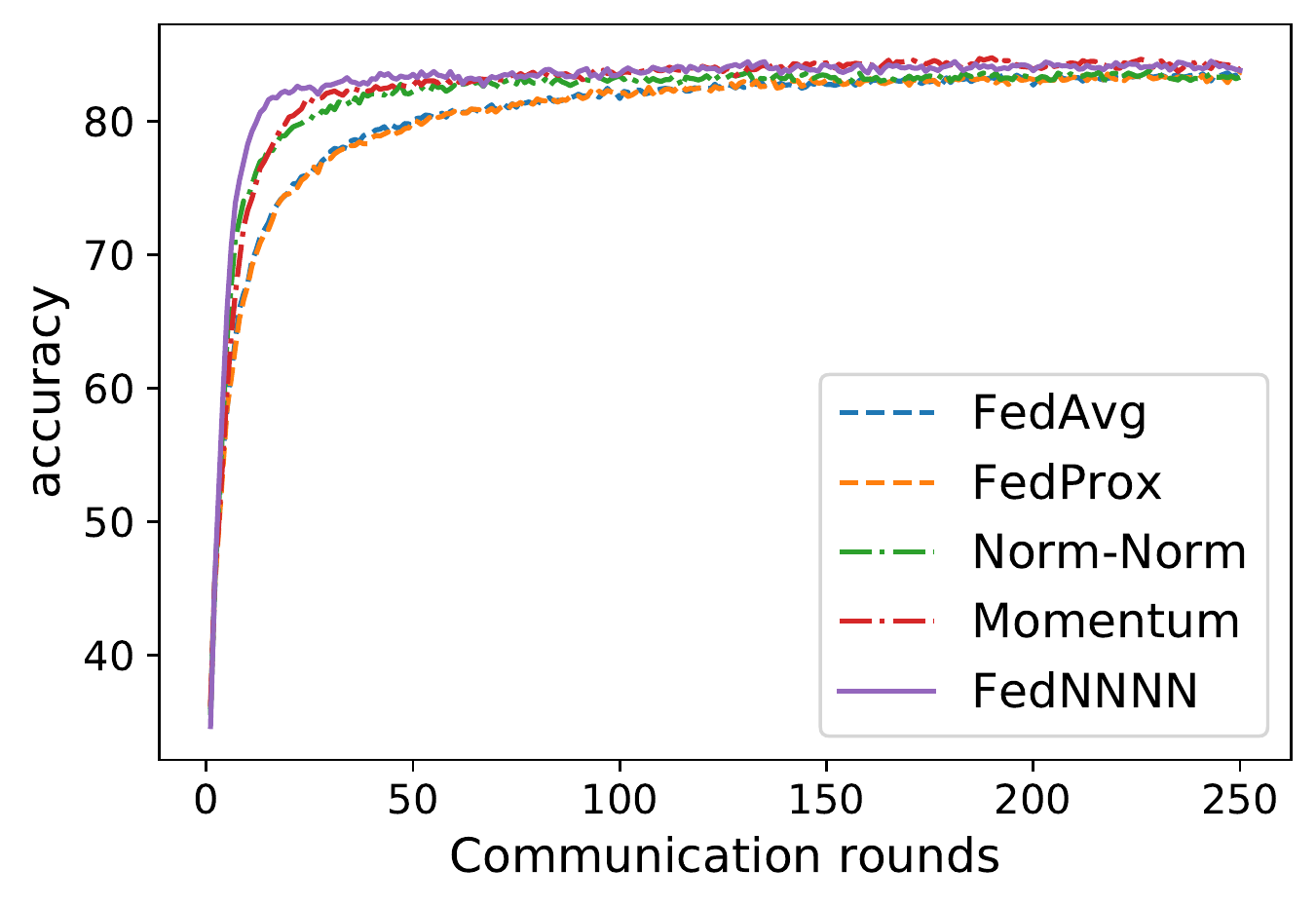}
        \caption{Accuracy curves on the CIFAR10 dataset under IID - UB condition.}
        \label{fig:cifar10_acc_iid-ub}       
      \end{minipage}
      \hspace{0.1cm}
      \begin{minipage}[t]{0.32\linewidth}
        \centering
        \includegraphics[width=\linewidth]{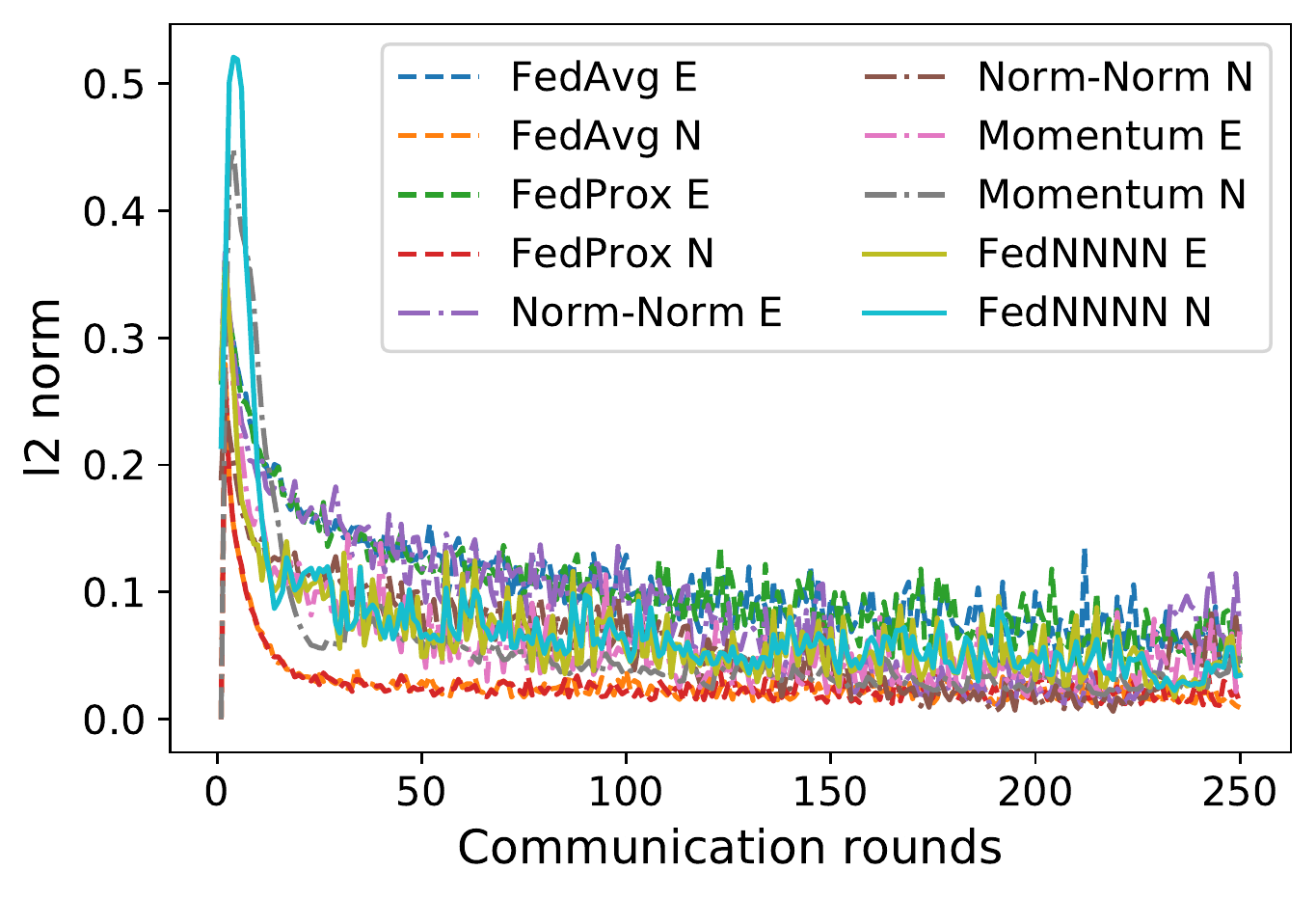}
        \caption{The $N$ and $E$ of fully connected layer calculated on the CIFAR10 dataset
        under IID - UB condition.}
        \label{fig:cifar10_norm_fc_iid-ub}       
      \end{minipage} 
      \hspace{0.1cm}
      \begin{minipage}[t]{0.32\linewidth}
        \centering
        \includegraphics[width=\linewidth]{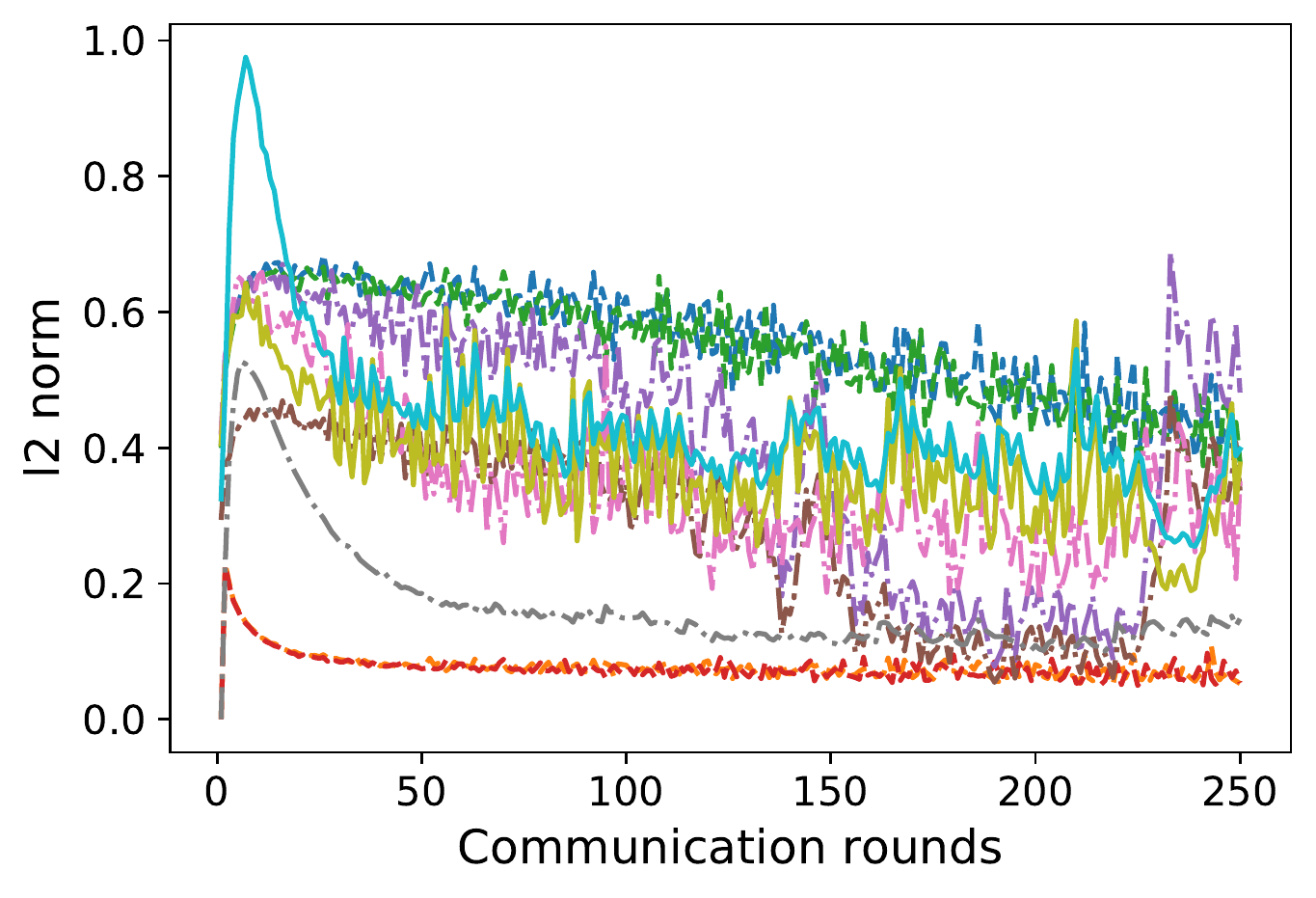}
        \caption{The $N$ and $E$ of convolutional layer calculated on the CIFAR10 dataset
        under IID - UB condition.}
        \label{fig:cifar10_norm_conv_iid-ub}       
      \end{minipage}\\
%
        \begin{minipage}[t]{0.32\linewidth}
        \centering
        \includegraphics[width=\linewidth]{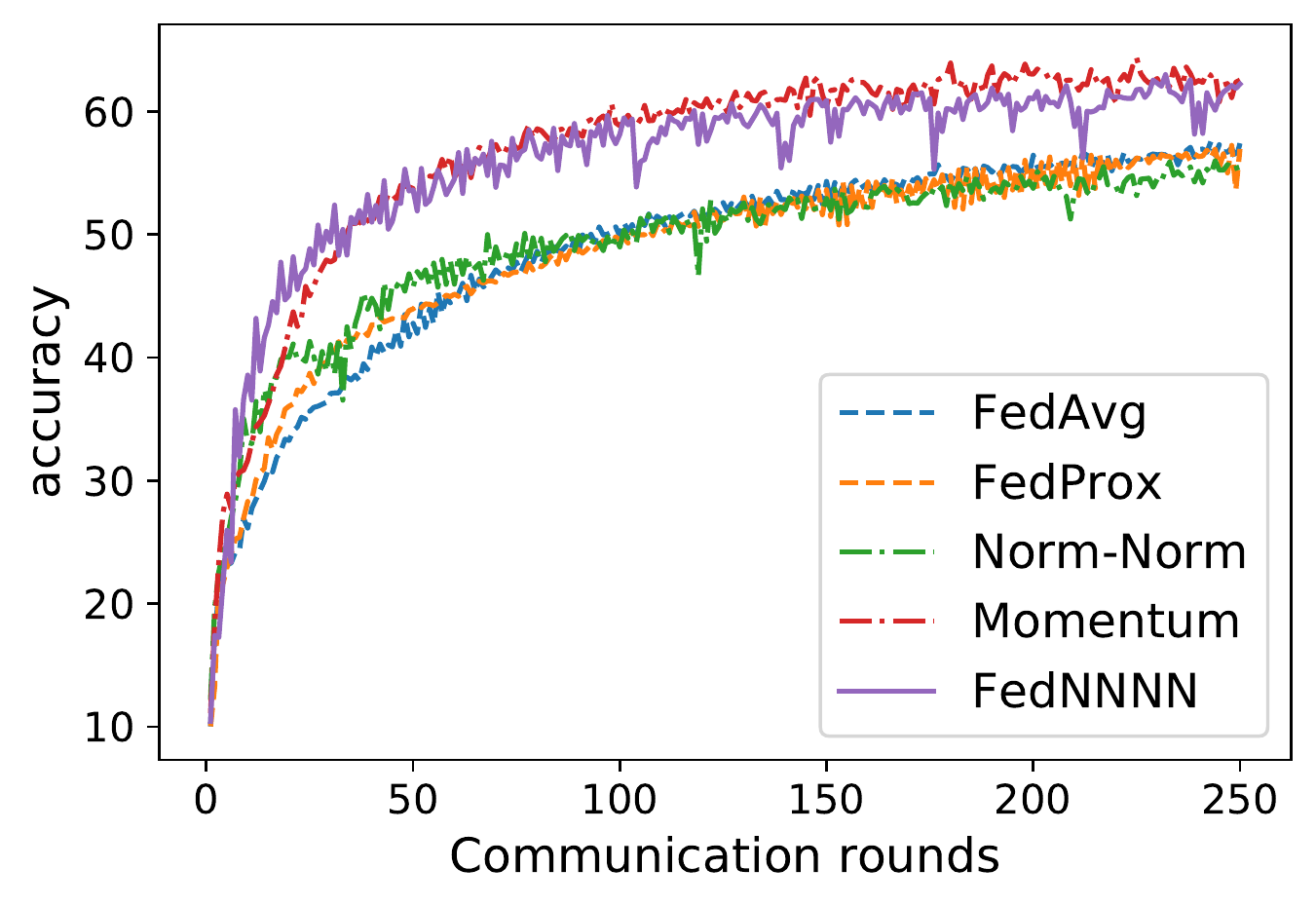}
        \caption{Accuracy curves on the CIFAR10 dataset under non-IID - UB condition.}
        \label{fig:cifar10_acc_niid-ub}       
      \end{minipage}
      \hspace{0.1cm}
      \begin{minipage}[t]{0.32\linewidth}
        \centering
        \includegraphics[width=\linewidth]{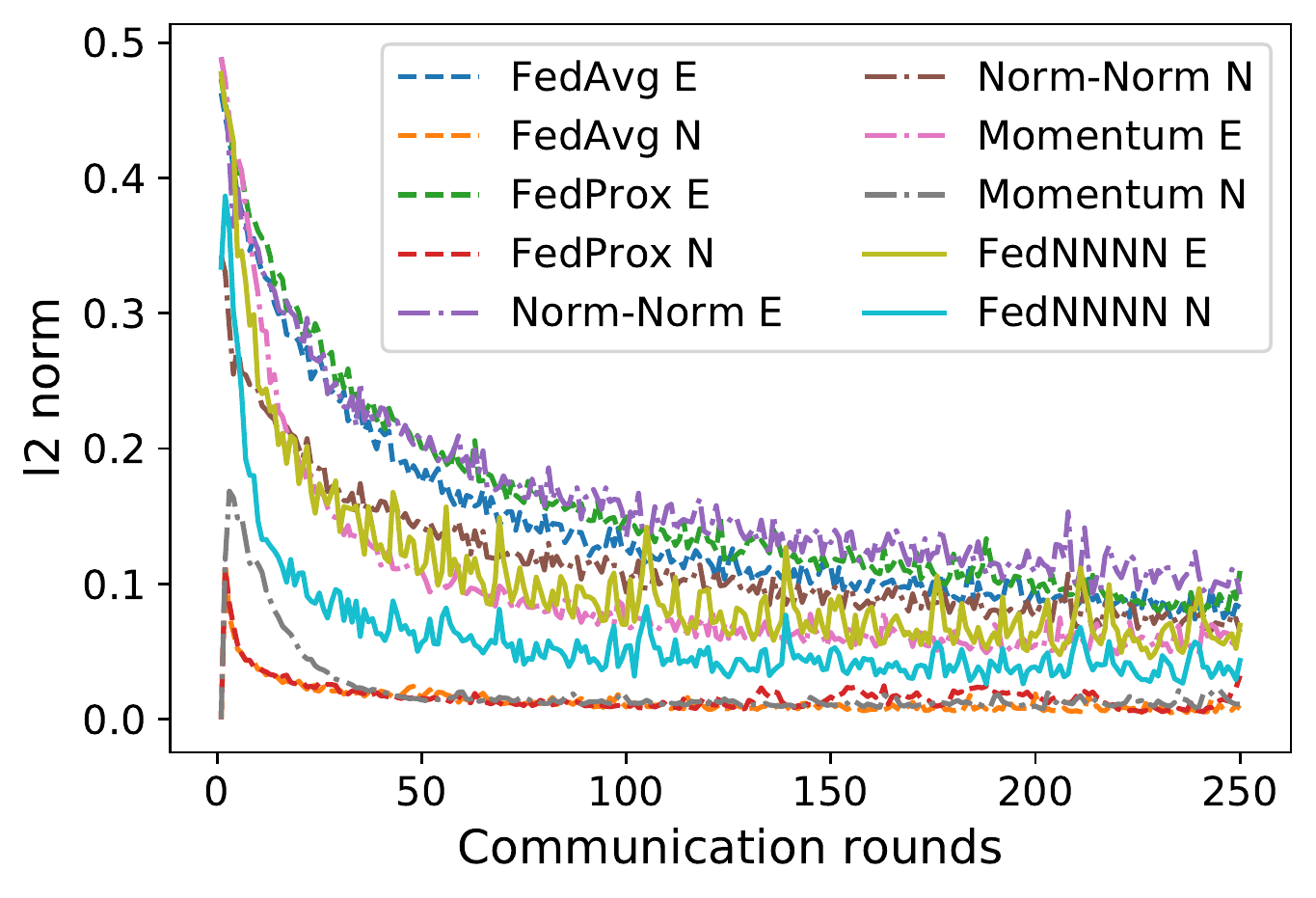}
        \caption{The $N$ and $E$ of fully connected layer calculated on the CIFAR10 dataset
        under non-IID - UB condition.}
        \label{fig:cifar10_norm_fc_niid-ub}       
      \end{minipage} 
      \hspace{0.1cm}
      \begin{minipage}[t]{0.32\linewidth}
        \centering
        \includegraphics[width=\linewidth]{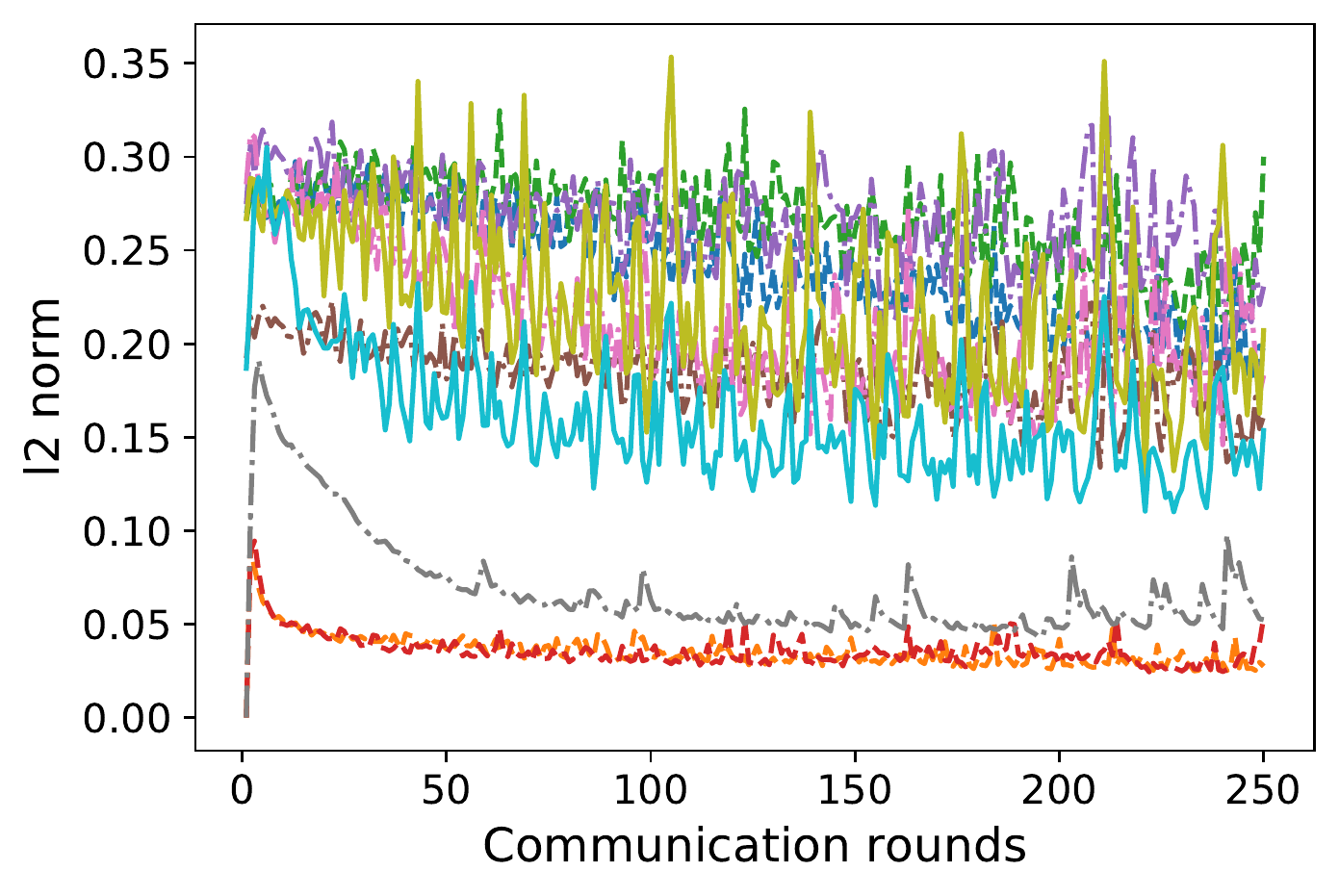}
        \caption{The $N$ and $E$ of convolutional layer calculated on the CIFAR10 dataset
        under non-IID - UB condition.}
        \label{fig:cifar10_norm_conv_niid-ub}       
      \end{minipage}
    \end{tabular}
\end{figure}

We plot the following $L_2$ norms in the figure.
\begin{eqnarray}
N&=&\norm*{\bm w_{t+1} - \bm w_t}\\
E&=&\sum_{k=1}^m\frac{n_k}{n}\norm*{\Delta \bm w^k_{t+1}}
\end{eqnarray}
The implications of the graphs on these norms can be summarized as follows:
\begin{description}
    \item[Improvement of convergence rate and accuracy]~\\ 
    The convergence of the proposed method is faster than that of the existing
    methods. The accuracy of the final result has also been improved. 
    
    \item[Trend of norm $N$]~\\
    The norm $N$ tends to be larger than that of the existing methods, 
    especially in
    the early stages of the communication round.
    This means that larger weight updates are achieved 
    in each round than those of the existing methods.
    
    \item[Trend of norm $E$]~\\ 
    On the other hand, the norm $E$ of the proposed method 
    tends to be smaller than the existing methods.
    This is due to the fact that clients are learning gradually 
    towards convergence.
    This trend is more pronounced in full connected layers, 
    but this is not the case in the convolutional layer.
    
    \item[Convolutional layer on CIFAR10]~\\ 
    Observing the norm of the convolutional layer in CIFAR10, the norm of the proposed method appears to be oscillating.
    It is believed that the weights are harder to converge
    than in full connected layers.
\end{description}

\end{document}